\documentclass[12pt, a4paper]{article}
\usepackage[utf8]{inputenc}
\usepackage[english]{babel}
\usepackage{geometry}
\usepackage{times}

\usepackage{etoolbox}
\apptocmd{\sloppy}{\hbadness 10000\relax}{}{}

\makeatletter
 \def\@textbottom{\vskip \z@ \@plus 20pt}
 \let\@texttop\relax
\makeatother

\usepackage[titletoc,toc,title]{appendix}

\usepackage{graphicx} 
\usepackage{subfigure}
\usepackage{float}
\usepackage{placeins}
\usepackage{caption}

\usepackage[round]{natbib}
\usepackage{tikz}
\usepackage{tikzscale}

\usepackage{array}
\usepackage{amsmath}
\usepackage{amsthm}
\usepackage{amsfonts}
\usepackage{amssymb}
\usepackage{mathtools}
\usepackage{dsfont}
\usepackage{cellspace}

\usepackage{url}
\usepackage[bookmarks=false]{hyperref}
\usepackage{cleveref}

\usepackage{natbib}

\usepackage[shortlabels]{enumitem}
\usepackage{commands}

\title{Random Fourier Features for Operator-Valued Kernels}

\author{
{\bf Romain Brault}\thanks{\url{ro.brault@telecom-paristech.fr}} \\
IBISC. \\
Universit\'e d'\'Evry val d'Essonne,\\
15 Boulevard de France, \'Evry, 91000, France.
\\ $\And$ \\
{\bf Florence d'Alch\'e-Buc}\thanks{\url{florence.dalche@telecom-paristech.fr}} \\
LTCI. \\
T\'el\'ecom ParisTech.\\
Universit\'e Paris-Saclay, \\
46 rue Barrault, Paris, 75684 cedex 13, France.
\\ $\And$ \\
{\bf Markus Heinonen}\thanks{\url{markus.o.heinonen@aalto.fi}} \\
Department of Information and Computer Science. \\
Aalto University,\\
FI-00076 Aalto, PO Box 15400, Finland.
}

\begin{document}

\maketitle

\begin{abstract}
Devoted to multi-task learning and structured output learning, operator-valued kernels provide a flexible tool to build vector-valued functions in the context of Reproducing Kernel Hilbert Spaces. To scale up these methods, we extend the celebrated Random Fourier Feature methodology to get an approximation of operator-valued kernels. We propose a general principle for Operator-valued Random Fourier Feature construction relying on a generalization of Bochner's theorem for translation-invariant operator-valued Mercer kernels. We prove the uniform convergence of the kernel approximation for bounded and unbounded operator random Fourier features using appropriate Bernstein matrix concentration inequality. An experimental proof-of-concept shows the quality of the approximation and the efficiency of the corresponding linear models on example datasets.
\end{abstract}

\section{Introduction}
Multi-task regression \citep{Micchelli2005}, structured classification \citep{Dinuzzo2011}, vector field learning \citep{Baldassare2012} and vector autoregression \citep{Sindhwani2013,Lim2015} are all learning problems that boil down to learning a vector while taking into account an appropriate output structure. A $p$-dimensional vector-valued model can account for couplings between the outputs for improved performance in comparison to $p$ independent scalar-valued models.
In this paper we are interested in a general and flexible approach to efficiently implement and learn vector-valued functions, while allowing couplings between the outputs. To achieve this goal, we turn to shallow architectures, namely the product of a (nonlinear) feature matrix $\tilde{\Phi}(x)$ and a parameter vector $\param$ such that $\tilde{f}(x) = \tilde{\Phi}(x)^* \param$, and combine two appealing methodologies: Operator-Valued Kernel Regression and Random Fourier Features.
 \par
Operator-Valued Kernels \citep{Micchelli2005,Carmeli2010,Alvarez2012} extend the classic scalar-valued kernels to vector-valued functions. As in the scalar case, operator-valued kernels (OVKs) are used to build Reproducing Kernel Hilbert Spaces (RKHS) in which representer theorems apply as for ridge regression or other appropriate loss functional. In these cases, learning a model in the RKHS boils down to learning a function of the form $f(x)=\sum_{i=1}^n K(x,x_i)\alpha_i$ where $x_1, \ldots, x_n$ are the training input data and each $\alpha_i, i=1, \ldots, n$ is a vector of the output space $\bmY$ and each $K(x,x_i)$, an operator on vectors of $\bmY$.
However, OVKs suffer from the same drawback as classic kernel machines: they scale poorly to very large datasets because they are very demanding in terms of memory and computation. Therefore, focusing on the case $\bmY=\mathbb{R}^p$, we propose to approximate OVKs by extending a methodology called Random Fourier Features (RFFs) \citep{Rahimi2007, Le2013, Alacarte, sriper2015, Bach2015, sutherland2015} so far developed to speed up scalar-valued kernel machines. The RFF approach linearizes a shift-invariant kernel model by generating explicitly an approximated feature map $\tilde{\phi}$. RFFs has been shown to be efficient on large datasets and further improved by efficient matrix computations of FastFood \citep{Le2013}, and is considered as one of the best large scale implementations of kernel methods, along with N\"ystrom approaches \citep{Yang2012}.
 \par
In this paper, we propose general Random Fourier Features for functions in vector-valued RKHS. Here are our contributions: (1) we define a general form of Operator Random Fourier Feature (ORFF) map for shift-invariant operator-valued kernels, (2) we construct explicit operator feature maps for a simple bounded kernel, the decomposable kernel, and more complex unbounded kernels curl-free and divergence-free kernels, (3) the corresponding kernel approximation is shown to uniformly converge towards the target kernel using appropriate Bernstein matrix concentration inequality, for both bounded and unbounded operator-valued kernels and (4) we illustrate the theoretical approach by a few numerical results.
 \par
The paper is organized as follows. In \cref{sec:background}, we recall Random Fourier Feature and Operator-valued kernels. In \cref{sec:orff}, we use extension of Bochner's theorem to propose a general principle of Operator Random Fourier Features and provide examples for decomposable, curl-free and divergence-free kernels.
In \cref{sec:concentration}, we present a theorem of uniform convergence for bounded and unbounded ORFFs (proof is given in \cref{subsec:concentration_proof}) and the conditions of its application. \Cref{sec:learning_with_ORFF} shows an numerical illustration on learning linear ORFF-models. \Cref{sec:conclusion} concludes the paper. The main proofs of the paper are presented in Appendix.

\subsection{Notations}
The euclidean inner product in $\mathbb{R}^d$ is denoted $\inner{\cdot, \cdot}$ and the euclidean norm is denoted $\norm{\cdot}$. The unit pure imaginary number $\sqrt{-1}$ is denoted $i$.
$\mathcal{B}(\mathbb{R}^d)$ is the Borel $\sigma$-algebra on $\mathbb{R}^d$.
For a function $f: \mathbb{R}^d \rightarrow \mathbb{R}$, if $dx$ is the Lebesgue measure on $\mathbb{R}^d$, we denote $\FT{f}$ its Fourier transform defined by:
\begin{equation*}
    \forall x \in \mathbb{R}^d, \FT{f}(x)=\int_{\mathbb{R}^d} e^{-i \inner{\omega,x}}f(\omega)d\omega.
\end{equation*}
The inverse Fourier transform of a function $g$ is defined as
\begin{equation*}
    \IFT{g}(\omega)= \int_{\mathbb{R}^d} e^{i \inner{x, \omega}}f(\omega)dx.
\end{equation*}
It is common to define the Fourier transform of a (positive) measure $\mu$ by
\begin{equation*}
    \FT{\mu}(x) = \int_{\mathbb{R}^d} e^{-i \inner{\omega,x}} d\mu(\omega).
\end{equation*}
If $\bmX$ and $\bmY$ are two vector spaces, we denote by $\mathcal{F}(\mathcal{X};\mathcal{Y})$ the vector space of functions $f:\mathcal{X}\to\mathcal{Y}$ and $\mathcal{C}(\mathcal{X};\mathcal{Y})\subset\mathcal{F}(\mathcal{X};\mathcal{Y})$ the subspace of continuous functions. If $\mathcal{H}$ is an Hilbert space we denote its scalar product by $\inner{.,.}_\mathcal{H}$ and its norm by $\norm{.}_\mathcal{H}$. We set $\mathcal{L}(\mathcal{H})=\mathcal{L}(\mathcal{H};\mathcal{H})$ to be the space of linear operators from $\mathcal{H}$ to itself. If $W\in\mathcal{L}(\mathcal{H})$, $\Ker W$ denotes the nullspace, $\Ima W$ the image and $W^\adjoint \in \mathcal{L}(\mathcal{H})$ the adjoint operator (transpose in the real case).
\subsection{Background} \label{sec:background}
\paragraph{Random Fourier Features:}
we first consider scalar-valued functions. Denote $k: \mathbb{R}^d \times \mathbb{R}^d \rightarrow \mathbb{R}$ a positive definite kernel on $\mathbb{R}^d$. A kernel $k$ is said to be \emph{shift-invariant} for the addition if for any $a \in \mathbb{R}^d$, $
\forall (x,x') \in \mathbb{R}^d \times \mathbb{R}^d, k(x-a,z-a) = k(x,z)$.
Then, we define $k_0: \mathbb{R}^d \rightarrow \mathbb{R}$ the function such that $k(x,z)= k_0(x-z)$. $k_0$ is called the \emph{signature} of kernel $k$. Bochner theorem is the theoretical result that leads to the Random Fourier Features.
\begin{theorem}[Bochner's theorem\footnote{See \citet{Rudin}.}]\label{th:bochner-scalar}
Every positive definite complex valued function is the Fourier transform of a non-negative measure. This implies that any positive definite, continuous and shift-invariant kernel $k$ is the Fourier transform of a non-negative measure $\mu$:
\begin{equation}\label{bochner-scalar}
k(x,z)=k_0(x-z) = \int_{\mathbb{R}^d} e^{-i \inner{\omega,x - z}} d\mu(\omega).
\end{equation}
 \end{theorem}
Without loss of generality for the Random Fourier methodology, we assume that $\mu$ is a probability measure, i.e. $\int_{\mathbb{R}^d} d\mu(\omega)=1$.
Then we can write \cref{bochner-scalar} as an expectation over $\mu$: $k_0(x-z) = \expectation_{\mu}\left[e^{-i \inner{\omega,x - z}}\right]$.
Both $k$ and $\mu$ are real-valued, and the imaginary part is null if and only if $\mu(\omega)= \mu(-\omega)$. We thus only write the real part:
\begin{equation*}
\begin{aligned}
k(x,z)&=\expectation_{\mu}[\cos \inner{\omega,x - z}]\label{eq:noyau-scal}\\
&=\expectation_{\mu}\left[ \cos \inner{\omega,z} \cos \inner{\omega,x} + \sin \inner{\omega,z} \sin \inner{\omega,x}\right].
\end{aligned}
\end{equation*}
Let $\Vect_{j=1}^D x_j$ denote the $Dm$-length column vector obtained by stacking vectors $x_j \in \mathbb{R}^m$. The feature map $\tilde{\phi}: \mathbb{R}^d \rightarrow \mathbb{R}^{2D}$ defined as
\begin{equation}\label{eq:rff}
\begin{aligned}
\tilde{\phi}(x)=\frac{1}{\sqrt{D}}\Vect_{j=1}^D\begin{pmatrix}\cos{\inner{x,\omega_j}} \\ \sin{\inner{x,\omega_j}}\end{pmatrix}, \enskip \omega_j \sim \mu
\end{aligned}
\end{equation}
is called a \emph{Random Fourier Feature map}. Each $\omega_{j}, j=1, \ldots, D$ is independently sampled from the inverse Fourier transform $\mu$ of $k_0$.
This Random Fourier Feature map provides the following Monte-Carlo estimator of the kernel:
\begin{equation}
\tilde{K}(x, z) = \tilde{\phi}(x)^* \tilde{\phi}(z),
\end{equation}
that is proven to uniformly converge towards the true kernel described in \cref{bochner-scalar}. The dimension $D$ governs the precision of this approximation whose uniform convergence towards the target kernel can be found in \citet{Rahimi2007} and in more recent papers with some refinements \citet{sutherland2015, sriper2015}.
Finally, it is important to notice that Random Fourier Feature approach \emph{only} requires two steps before learning: (1) define the inverse Fourier transform of the given shift-invariant kernel, (2) compute the randomized feature map using the spectral distribution $\mu$. For the Gaussian kernel $k(x-z) = exp(-\gamma \norm{x - z}^2)$, the spectral distribution $\mu(\omega)$ is Gaussian \citet{Rahimi2007}.
\paragraph{Operator-valued kernels:}
we now turn to vector-valued functions and consider vector-valued Reproducing Kernel Hilbert spaces (vv-RKHS) theory. The definitions are given for input space $\bmX \subset \mathbb{C}^d$ and output space $\bmY \subset \mathbb{C}^p$. We will define operator-valued kernel as reproducing kernels following the presentation of \citet{Carmeli2010}.
Given $\bmX$ and $\bmY$, a map $K:\bmX\times\bmX\to\bmL(\bmY)$ is called a $\bmY$-reproducing kernel if
\begin{equation*}
  \sum_{i,j=1}^N\inner{K(x_i,x_j)y_j,y_i}\ge0,
\end{equation*}
for all $x_1,\hdots,x_N$ in $\bmX$, all $y_1,\hdots,y_N$ in $\bmY$ and $N\ge1$. Given $x\in\bmX$, $K_x:\bmY\to\mathcal{F}(\bmX;\bmY)$ denotes the linear operator whose action on a vector $y$ is the function $K_xy\in\mathcal{F}(\bmX;\bmY)$ defined by
 $(K_x y)(z)=K(z,x)y, \enskip \forall z\in\bmX$.
 \par
 Additionally, given a $\bmY$-reproducing kernel $K$, there is a unique Hilbert space $\bmH_K\subset\mathcal{F}(\bmX;\bmY)$ satisfying $K_x\in\bmL(\bmY;\bmH_K), \enskip \forall x\in\bmX$ and $f(x)=K^\adjoint_x f, \enskip \forall x\in\bmX, \forall f\in\bmH_K$, where $K^\adjoint_x:\bmH_K\to\bmY$ is the adjoint of $K_x$.
The space $\bmH_K$ is called the \emph{(vector-valued) Reproducing Kernel Hilbert Space} associated with $K$. The corres\-ponding product and norm are denoted by $\inner{.,.}_K$ and $\norm{.}_K$, respectively. As a consequence \citep{Carmeli2010} we have:
\begin{equation*}
\begin{aligned}
K(x,z)&=K^\adjoint_x K_z \enskip\forall x,z\in\bmX \\
\bmH_K&=\lspan\left\{ K_x y \enskip\middle|\enskip \forall x\in\bmX,\enskip\forall y\in\bmY \right\}
\end{aligned}
\end{equation*}
Another way to describe functions of $\bmH_K$ consists in using a suitable feature map.
\begin{proposition}[\citet{Carmeli2010}]
\label{pr:feature_operator}
Let $\bmH$ be a Hilbert space and $\Phi:\bmX\to\mathcal{L}(\bmY;\bmH)$, with $\Phi_x\triangleq \Phi(x)$. Then the operator $W:\bmH\to\mathcal{F}(\bmX;\bmY)$ defined by $(W g)(x)=\Phi_x^\adjoint g, \enskip \forall g \in\bmH, \forall x\in\bmX$ is a partial isometry from $\bmH$ onto the reproducing kernel Hilbert space $\bmH_K$ with reproducing kernel
\begin{equation*}
K(x,z)=\Phi^\adjoint_x\Phi_z, \enskip \forall x, z\in\bmX.
\end{equation*}
$W^\adjoint W$ is the orthogonal projection onto \begin{equation*}
  \Ker W^\bot = \lspan\left\{ \Phi_x y \enskip\middle|\enskip \forall x\in\bmX,\enskip\forall y\in\bmY \right\}.
\end{equation*}
Then $\norm{f}_K=\inf\left\{ \norm{g}_\bmH \enskip\middle|\enskip \forall g \in\bmH,\enskip Wg=f \right\}$.
\end{proposition}
We call $\Phi$ a \emph{feature map}, $W$ a \emph{feature operator} and $\bmH$ a \emph{feature space}.
\par
In this paper, we are interested on finding feature maps of this form for shift-invariant $\mathbb{R}^p$-Mercer kernels using the following definitions.
A reproducing kernel $K$ on $\mathbb{R}^d$ is a $\mathbb{R}^p$-Mercer provided that $\bmH_K$ is a subspace of $\bmC(\mathbb{R}^d;\mathbb{R}^p)$. It is said to be a \emph{shift-invariant kernel} or a \emph{translation-invariant kernel} for the addition if $ K(x+a,z+a)=K(x,z), \forall (x,z,a) \in \bmX^3$. It is characterized by a function $K_0:\bmX \to \bmL(\bmY)$ of completely positive type such that $K(x,z)=K_0(\delta)$, with $\delta=x-z$.
\section{Operator-valued Random Fourier Features}\label{sec:orff}
\subsection{Spectral representation of shift-invariant vector-valued Mercer kernels}
The goal of this work is to build approximated matrix-valued feature map for shift-invariant $\mathbb{R}^p$-Mercer kernels, denoted with $K$, such that any function $f \in \bmH_K$ can be approximated by a function $\tilde{f}$ defined by:
 \begin{equation*}
 \tilde{f}(x) = \tilde{\Phi}(x)^* \theta
 \end{equation*}
where $\tilde{\Phi}(x)$ is a matrix of size $(m \times p)$ and $\theta$ is an $m$-dimensional vector.
We propose a randomized approximation of such a feature map using a generalization of the Bochner theorem for operator-valued functions. For this purpose, we build upon the work of \citet{Carmeli2010} that introduced the Fourier representation of shift-invariant Operator-Valued Mercer Kernels on locally compact Abelian groups $\bmX$ using the general framework of Pontryagin duality (see for instance \citet{folland1994course}). In a few words, Pontryagin duality deals with functions on locally compact Abelian groups, and allows to define their Fourier transform in a very general way. For sake of simplicity, we instantiate the general results of \citet{Carmeli2010, Zhang2012} for our case of interest of $\bmX = \mathbb{R}^d$ and $\bmY = \mathbb{R}^p$.
The following proposition extends Bochner's theorem to any shift-invariant $\mathbb{R}^p$-Mercer kernel.
\begin{proposition}[Operator-valued Bochner's theorem\footnote{Equation (36) in \citet{Zhang2012}.}]
A continuous function $K$ from $\mathbb{R}^d \times \mathbb{R}^d$ to $\mathcal{L}(\mathbb{R}^p)$ is a shift-invariant reproducing kernel if and only if $\forall x, z \in \mathbb{R}^d$, it is the Fourier transform of a positive operator-valued measure $\bmM: \mathcal{B}(\mathbb{R}^d) \to \bmL_+(\mathbb{R}^p)$:
\begin{equation*}
 K(x, z) = \int_{\mathbb{R}^d} e^{-i \inner{x - z, \omega}} d\bmM(\omega),
\end{equation*}
where $\bmM$ belongs to the set of all the $\bmL_+(\mathbb{R}^p)$-valued measures of bounded variation on the $\sigma$-algebra of Borel subsets of $\mathbb{R}^d$.
\end{proposition}
However it is much more convenient to use a more explicit result that involves real-valued (positive) measures. The following proposition instantiates Prop. 13 in \citet{Carmeli2010} to matrix-valued operators.
\begin{proposition}[\citet{Carmeli2010}]\label{pr:bochner}
Let $\mu$ be a positive measure on $\mathbb{R}^d$ and $A: \mathbb{R}^d\to \mathcal{L}(\mathbb{R}^p)$ such that $\inner{A(.)y,y'}\in L^1(\mathbb{R}^d,d\mu)$ for all $y,y'\in\mathbb{R}^p$ and $A(\omega)\ge0$ for $\mu$-almost all $\omega$. Then, for all $\delta \in \mathbb{R}^d$, $\forall \ell,m \inrange{1}{p}$,
\begin{equation}
\label{eq:AK0}
K_0(\delta)_{\ell m}=\int_{\mathbb{R}^d}e^{-i\inner{\delta,\omega}}A(\omega)_{\ell m}d\mu(\omega)
\end{equation}
is the kernel signature of a shift-invariant $\mathbb{R}^p$-Mercer kernel $K$ such that $K(x,z)=K_0(x-z)$. In other terms, each real-valued function $K_0(\cdot)_{\ell m}$ is the Fourier transform of $A(\cdot)_{\ell m}\density(\cdot)$ where $\density(\omega)=\frac{d\mu}{d\omega}$ is the Radon-Nikodym derivative of the measure $\mu$, which is also called the density of the measure $\mu$. Any shift-invariant kernel is of the above form for some pair $(A(\omega),\mu (\omega))$.
\end{proposition}
This theorem is proved in \citet{Carmeli2010}. When $p=1$ one can always assume $A$ is reduced to the scalar $1$, $\mu$ is still a bounded positive measure and we retrieve the Bochner theorem applied to the scalar case (\cref{th:bochner-scalar}).
Now we introduce the following proposition that directly is a direct consequence of \cref{pr:bochner}.
\begin{proposition}[Feature map]
Given the conditions of \cref{pr:bochner}, we define $B(\omega)$ such that $A(\omega) = B(\omega)B(\omega)^*$. Then the function $\Phi:\mathbb{R}^p \rightarrow L^2(\mathbb{R}^d,\mu;\mathbb{R}^p)$ defined for all $x \in \mathbb{R}^p$ by
\begin{equation}
\label{eq:feature_shiftinv_operator}
\forall y\in\mathbb{R}^p,\enskip \left(\Phi_x y \right)(\omega)=e^{i\inner{x,\omega}}B(\omega)^*y
\end{equation}
is a feature map of the shift-invariant kernel $K$, i.e. it satisfies for all $x,z$ in $\mathbb{R}^d$, $\Phi_x^*\Phi_z=K(x,z)$.
\end{proposition}
Thus, to define an approximation of a given operator-valued kernel, we need an inversion theorem that provides an explicit construction of the pair $A(\omega), \mu(\omega)$ from the kernel signature. Proposition 14 in \citet{Carmeli2010}, instantiated to $\mathbb{R}^p$-Mercer kernel gives the solution.
\begin{proposition}[\citet{Carmeli2010}]
\label{pr:inverse_ovk_Fourier_decomposition}
Let $K$ be a shift-invariant $\mathbb{R}^p$-Mercer kernel. Suppose that $\forall z \in \mathbb{R}^d, \forall y,y' \in\mathbb{R}^{p}$, $\inner{K_0(.)y,y'}\in L^1(\mathbb{R}^d,dx)$ where $dx$ denotes the Lebesgue measure. Define $C$ such that $\forall \omega \in \mathbb{R}^d, \forall \ell,m \inrange{1}{p}$,
\begin{equation}\label{eq:CK0}
C(\omega)_{\ell m} = \int_{\mathbb{R}^d} e^{i \inner{\delta,\omega}}K_0(\delta)_{\ell m}d\delta.
\end{equation}
Then
\begin{enumerate}[i)]
\item $C(\omega)$ is an non-negative matrix for all $\omega\in\mathbb{R}^d$,
\item $\inner{C(.)y,y'}\in L^1(\mathbb{R}^d,d\omega)$ for all $y,y'\in\mathbb{R}^p$,
\item for all $\delta\in\mathbb{R}^d$,$\forall \ell,m \inrange{1}{p}$,
\begin{equation*}
K_0(\delta)_{\ell m}=\int_{\mathbb{R}^d}e^{-i \inner{\delta,\omega}}C(\omega)_{\ell m}d\omega.
\end{equation*}
\end{enumerate}
\end{proposition}
From \cref{eq:AK0} and \cref{eq:CK0}, we can write the following equality concerning the matrix-valued kernel signature $K_0$, coefficient by coefficient: $\forall \delta \in \mathbb{R}^d, \forall i,j \inrange{1}{p},$
\begin{equation*}
\int_{\mathbb{R}^d}e^{-i\inner{\delta, \omega}}C(\omega)_{ij}d\omega=\int_{\mathbb{R}^d}e^{-i \inner{\delta, \omega}}A(\omega)_{ij}d\mu(\omega).
\end{equation*}
We then conclude that the following equality holds almost everywhere for $\omega \in \mathbb{R}^d$: $C(\omega)_{ij}=A(\omega)_{ij}p_{\mu}(\omega)$ where $p_{\mu}(\omega)=\frac{d\mu}{d\omega}$. Without loss of generality we assume that $\int_{\mathbb{R}^d} d\mu(\omega)=1$ and thus, $\mu$ is a probability distribution. Note that this is always possible through an appropriate normalization of the kernel. Then $p_{\mu}$ is the density of $\mu$. The \cref{pr:bochner} thus results in an expectation:
\begin{equation}
 K_0(x-z)= \expectation_{\mu}[e^{-i\inner{x-z,\omega}} A(\omega)]
\end{equation}
\subsection{Construction of Operator Random Fourier Feature}
Given a $\mathbb{R}^p$-Mercer shift-invariant kernel $K$ on $\mathbb{R}^d$, we build an Operator-Valued Random Fourier Feature (ORFF) map in three steps:
\begin{enumerate}[1)]
\item compute $C:\mathbb{R}^d \rightarrow \bmL(\mathbb{R}^p)$ from \cref{eq:CK0} by using the inverse Fourier transform (in the sense of \cref{pr:inverse_ovk_Fourier_decomposition}) of $K_0$, the signature of $K$;
\item find $A(\omega)$, $\density(\omega)$ and compute $B(\omega)$ such that $A(\omega)=B(\omega)B(\omega)^*$;
\item build an randomized feature map via Monte-Carlo sampling from the probability measure $\mu$ and the application $B$.
\end{enumerate}
\subsection{Monte-Carlo approximation}
Let $\vect_{j=1}^D X_j$ denote the block matrix of size $rD \times s$ obtained by stacking $D$ matrices $X_1, \ldots, X_D$ of size $r \times s$.
Assuming steps 1 and 2 have been performed, for all $j=1, \ldots, n$, we find a decomposition $A(\omega_j)=B(\omega_j)B(\omega_j)^*$ either by exhibiting a general analytical closed-form or using a numerical decomposition. Denote $p \times p'$ the dimension of the matrix $B(\omega)$. We then propose a randomized matrix-valued feature map: $\forall x \in \mathbb{R}^d$,
\begin{equation}
\label{eq:OV_RFF1}
\begin{aligned}
\tilde{\Phi}(x)
	&=\frac{1}{\sqrt{D}}\Vect_{j=1}^D \Phi_x(\omega_j), \enskip\omega_j \sim \mu \\
	&=\frac{1}{\sqrt{D}}\Vect_{j=1}^D e^{-i \inner{x,\omega_j}}B(\omega_j)^*, \enskip\omega_j \sim \mu.
\end{aligned}
\end{equation}
The corresponding approximation for the kernel is then: $\forall x,z \in \mathbb{R}^d$
\begin{equation*}
\begin{aligned}
\tilde{K}(x,z)&=\tilde{\Phi}(x)^*\tilde{\Phi}(z)\\
&=\frac{1}{D}\sum\nolimits_{j=1}^D e^{-i \inner{x,\omega_j}}B(\omega_j)e^{i \inner{z,\omega_j}}B(\omega_j)^*\\
&=\frac{1}{D}\sum\nolimits_{j=1}^D e^{-i\inner{x-z,\omega_j}}A(\omega_j).
\end{aligned}
\end{equation*}
The Monte-Carlo estimator $\tilde{\Phi}(x)^*\tilde{\Phi}(z)$ 
converges in probability to $K(x,z)$ when $D$ tends to infinity. Namely,
\begin{equation*}
    \tilde{K}(x,z)=\tilde{\Phi}(x)^*\tilde{\Phi}(z) \xrightarrow[D\to\infty]{p.} \expectation_{\mu}\left[ e^{-i\inner{x-z,\omega}}A(\omega) \right]=K(x,z)
\end{equation*}
As for the scalar-valued kernel, a real-valued matrix-valued function has a real matrix-valued Fourier transform if $A(\omega)$ is even with respect to $\omega$. Taking this point into account, we define the feature map of a real matrix-valued kernel as
\begin{equation*}
\begin{aligned}
\tilde{\Phi}(x)=\frac{1}{\sqrt{D}}\Vect_{j=1}^D\begin{pmatrix}\cos{\inner{x,\omega_j}}B(\omega_j)^\adjoint \\ \sin{\inner{x,\omega_j}}B(\omega_j)^\adjoint\end{pmatrix}, \enskip \omega_j \sim \mu.
\end{aligned}
\end{equation*}
The kernel approximation becomes
\begin{equation*}
\begin{aligned}
\tilde{\Phi}(x)^*\tilde{\Phi}(z) &= \frac{1}{D} \sum_{j=1}^D \substack{ \cos{\inner{x,\omega_j}}\cos{\inner{z,\omega_j}} A(\omega_j) \\ \sin{\inner{x,\omega_j}}\sin{\inner{z,\omega_j}} A(\omega_j) }+ \\
&= \frac{1}{D} \sum_{j=1}^D \cos{\inner{x-z,\omega_j}}A(\omega_j).
\end{aligned}
\end{equation*}
In the following, we give an explicit construction of ORFFs for three well-known $\mathbb{R}^p$-Mercer and shift-invariant kernels: the \emph{decomposable kernel} introduced in \citet{Micchelli2005} for multi-task regression and the \emph{curl-free} and the \emph{divergence-free} kernels studied in \citet{Macedo2008, Baldassare2012} for vector field learning. All these kernels are defined using a scalar-valued shift-invariant Mercer kernel $k: \mathbb{R}^d \times \mathbb{R}^d \to \mathbb{R}$ whose signature is denoted $k_0$. A usual choice is to choose $k$ as a Gaussian kernel with $k_0(\delta) = \exp\left(-\frac{\norm{\delta}^2}{2\sigma^2}\right)$, which gives $\mu= \bmN(0,\sigma^{-2}I)$ \citep{huang2013random} as its inverse Fourier transform.
\begin{definition}[Decomposable kernel]
\label{dec-kernel}
Let A be a $(p \times p)$ positive semi-definite matrix. $K$ defined as $\forall (x,z) \in \mathbb{R}^d \times \mathbb{R}^d, K(x,z) = k(x,z)A$ is a $\mathbb{R}^p$-Mercer shift-invariant reproducing kernel.
\end{definition}
Matrix $A$ encodes the relationships between the outputs coordinates. If a graph coding for the proximity between tasks is known, then it is shown in \citet{Evgeniou2005, Baldassare2010} that $A$ can be chosen equal to the pseudo inverse $L^{\dagger}$ of the graph Laplacian, and then the $\ell_2$ norm in $\bmH_K$ is a graph-regularizing penalty for the outputs (tasks). When no prior knowledge is available, $A$ can be set to the empirical covariance of the output training data or learned with one of the algorithms proposed in the literature \citep{Dinuzzo2011, Sindhwani2013, Lim2015}. Another interesting property of the decomposable kernel is its universality. A reproducing kernel $K$ is said \emph{universal} if the associated RKHS $\bmH_K$ is dense in the space $\bmC(\bmX,\bmY)$.
\begin{example}[ORFF for decomposable kernel]
\begin{equation*}
C^{dec}(\omega)_{\ell m}=\int_{\mathcal{X}}e^{i\inner{\delta,\omega}} k_0(\delta)A_{\ell m} d\delta = A_{\ell m}\IFT{k_0}(\omega)\\
\end{equation*}
Hence, $A^{dec}(\omega)=A$ and $p_{\mu}^{dec}(\omega)=\IFT{k_0}(\omega)$.
\end{example}
\paragraph{ORFF for curl-free and div-free ker\-nels:}
Curl-free and divergence-free kernels provide an interesting application of operator-valued kernels \citep{Macedo2008, Baldassare2012, Micheli2013} to \emph{vector field} learning, for which input and output spaces have the same dimensions ($d=p$). Applications cover shape deformation analysis \citep{Micheli2013} and magnetic fields approximations \citep{Wahlstrom2013}. These kernels discussed in \citet{Fuselier2006} allow encoding input-dependent similarities between vector-fields.
\begin{definition}[Curl-free and Div-free kernel]\label{curl-div-free}
We have $d=p$.
The \emph{divergence-free} kernel is defined as
\begin{equation*}\label{div-def}
K^{div}(x,z)=K^{div}_0(\delta)
= (\nabla\nabla^* - \Delta I) k_0(\delta)
\end{equation*}
and the \emph{curl-free} kernel as
\begin{equation*}\label{curl-def}
K^{curl}(x,z)=K_0^{curl}(\delta)=-\nabla\nabla^* k_0(\delta),
\end{equation*}
where $\nabla\nabla^*$ is the Hessian operator and $\Delta$ is the Laplacian operator.
\end{definition}
Although taken separately these kernels are not universal, a convex combination of the curl-free and divergence-free kernels allows to learn any vector field that satisfies the Helmholtz decomposition theorem \citep{Macedo2008, Baldassare2012}. For the divergence-free and curl-free kernel we use the differentiation properties of the Fourier transform.
\begin{example}[ORFF for curl-free kernel:] $\forall \ell,m \inrange{1}{p}$,
\begin{equation*}
\begin{aligned}
C^{curl}(\omega)_{\ell m}&=-\IFT{\frac{\partial}{\partial \delta_{\ell}}\frac{\partial}{\partial \delta_{m}}k_0}(\omega)\\
     &= \omega_{\ell}\omega_m \IFT{k_0}(\omega)
\end{aligned}
\end{equation*}
Hence, $A^{curl}(\omega)=\omega\omega^*$ and $p_{\mu}^{curl}(\omega)=\IFT{k_0}(\omega)$. We can obtain directly: $B^{curl}(\omega)=\omega$.
\end{example}
For the divergence-free kernel we first compute the Fourier transform of the Laplacian of a scalar kernel using differentiation and linearity properties of the Fourier transform. We denote $\delta_{\{\ell=m\}}$ as the Kronecker delta which is $1$ if $\ell=m$ and zero otherwise.
\begin{example}[ORFF for divergence-free kernel:]
\begin{equation*}
\begin{aligned}
 C^{div}(\omega)_{\ell m}&=\IFT{\frac{\partial}{\partial \delta_{\ell}}\frac{\partial}{\partial \delta_{m}}k_0-\delta_{\{\ell=m\}}\Delta k_0} \\
     &=\IFT{\frac{\partial}{\partial \delta_{\ell}}\frac{\partial}{\partial \delta_{m}}k_0}-\delta_{\{\ell=m\}}\IFT{\Delta k_0} \\
     &=(\delta_{\{\ell=m\}}-\omega_{\ell}\omega_{m})\norm{\omega}_2^2\IFT{k_0},
\end{aligned}
\end{equation*}
since
\begin{equation*}
  \IFT{\Delta k_0(\delta)}=\sum_{k=1}^p\IFT{\frac{\partial}{\partial\delta_{k}}k_0}=-\norm{\omega}_2^2\IFT{k_0}.
\end{equation*}
Hence $A^{div}(\omega)=I\norm{\omega}_2^2-\omega\omega^*$ and $p_{\mu}^{div}(\omega)=\IFT{k_0}(\omega)$. Here, $B^{div}(\omega)$ has to be obtained by a numerical decomposition such as Cholesky or SVD.
\end{example}

\begin{figure}
 \centering
 \includegraphics[width=0.9\textwidth]{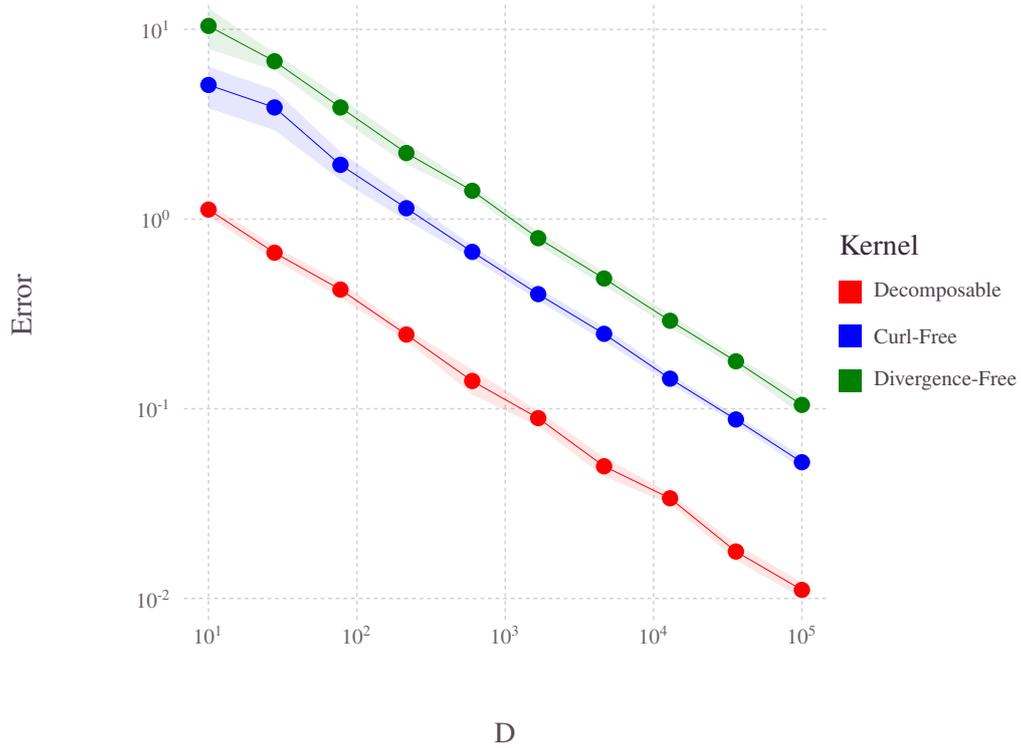}
 \caption{Empirical Approximation Error versus number of random features $D$ induced by the ORFF approximation for different operator-valued kernels}
 \label{fig:approximation_error}
\end{figure}

\section{Uniform error bound on ORFF approximation}\label{sec:concentration}
We are now interested on measuring how close the approximation $\tilde{K}(x,z)=\tilde{\Phi}(x)^*\tilde{\Phi}(z)$ is close to the target kernel $K(x,z)$ for any $x,z$ in a compact set $\mathcal{C}$. If $A$ is a real matrix, we denote $\norm{A}_2$ its spectral norm, defined as the square root of the largest eigenvalue of $A$. For $x$ and $z$ in some compact $\bmC \subset \mathbb{R}^d$, we consider: $F(x-z) =\tilde{K}(x,z)-K(x,z)$ and study how the uniform norm
\begin{equation}\label{eq:norm_inf}
  \norm{F}_{\infty}=\sup_{x,z\in\mathcal{C}}\norm{\tilde{K}(x,z)-K(x,z)}_2
\end{equation}
behaves according to $D$. \Cref{fig:approximation_error} empirically shows convergence of three different OVK approximations for $x,z$ from the compact $[-1,1]^4$ using an increasing number of sample points $D$. The log-log plot shows that all three kernels have the same convergence rate, up to a multiplicative factor.
\par
In order to bound the error with high probability, we turn to concentration inequalities devoted to random matrices \citep{Boucheron}. In the case of the decomposable kernel, the answer to that question can be directly obtained as a consequence of the uniform convergence of RFFs in the scalar case obtained by \citet{Rahimi2007} and other authors \citep{sutherland2015, sriper2015} since in this case,
\begin{equation*}
\norm{\tilde{K}(x,z)-K(x,z)}_2 = \norm{A}_2\norm{\tilde{k}(x,z) - k(x,z)}
\end{equation*}
This theorem and its proof are presented in \cref{sec:dec-bound}.
\par
More interestingly, we propose a new bound for Operator Random Fourier Feature approximation in the general case. It relies on three main ideas: (i) Matrix concentration inequality for random matrices has to be used instead of concentration inequality for (scalar) random variables, (ii) Instead of using Hoeffding inequality as in the scalar case (proof of \citet{Rahimi2007}) but for matrix concentration \citep{Mackey2014} we use a refined inequality such as the Bernstein matrix inequality \citep{Ahls2002,Boucheron,Tropp}, also used for the scalar case in \citep{sutherland2015}, (iii) we propose a general theorem valid for random matrices with bounded norms (case for decomposable kernel ORFF approximation) as well as with unbounded norms (curl and divergence-free kernels). For the latter, we notice that their norms behave as subexponential random variables \citep{Koltchinskii2013remark}. Before introducing the new theorem, we give the definition of the Orlicz norm and subexponential random variables.
\begin{definition}[Orlicz norm]
We follow the definition given by \citet{Koltchinskii2013remark}.
Let $\psi:\mathbb{R}_+\to\mathbb{R}_+$ be a non-decreasing convex function with $\psi(0)=0$. For a random variable $X$ on a measured space $(\Omega,\mathcal{T}(\Omega),\mu)$,
\begin{equation*}
  \norm{X}_{\psi} \triangleq \inf \left\{ C > 0 \,\, \middle| \,\, \expectation[\psi\left( \abs{X}/C \right)]\le 1 \right\}.
\end{equation*}
\end{definition}
Here, the function $\psi$ is chosen as $\psi(u)=\psi_{\alpha}(u)$ where $\psi_{\alpha}(u)\triangleq e^{u^{\alpha}}-1$. When $\alpha=1$, a random variable with finite Orlicz norm is called a \emph{subexponential variable} because its tails decrease at least exponentially fast.
\begin{theorem}
\label{pr:ov-rff_concentration}
Let $\mathcal{C}$ be a compact subset of $\mathbb{R}^d$ of diameter $l$. Let $K$ be a shift-invariant $\mathbb{R}^p$-Mercer kernel on $\mathbb{R}^d$, $K_0$ its signature and $\density(\cdot)A(\cdot)$ the inverse Fourier transform of the kernel's signature (in the sense of \cref{pr:inverse_ovk_Fourier_decomposition}) where $\density$ is the density of a probability measure $\mu$ considering appropriate normalization.
Let $D$ be a positive integer and $\omega_1, \ldots, \omega_D$, i.i.d. random vectors drawn according to the probability law $\mu$. For $x,z \in \mathcal{C}$, we recall
\begin{equation*}
    \tilde{K}(x,z)=\sum_{j=1}^D\cos \inner{x-z,\omega_j} A(\omega_j).
\end{equation*}
We note for all $j \inrange{1}{D}$,
\begin{equation*}
    F_j(x-z)=\frac{1}{D}\left(\sum_{j=1}^D\cos \inner{x-z,\omega_j} A(\omega_j)- K(x,z)\right)
\end{equation*}
and $F(x-z)= \tilde{K}(x,z) - K(x,z)$. $\norm{F}_{\infty}$ denotes the infinite norm of $F(x-z)$ on the compact $\mathcal{C}$ as introduced in \cref{eq:norm_inf}.
If one can define the following terms $(b_D, m, \sigma_p^2)\in\mathbb{R}^3_+$:
\begin{align*}
b_D &= \sup_{x,z\in\mathcal{C}}D\norm{\expectation_{\mu}\left[\sum_{j=1}^D \left(F_j(x-z)\right)^2\right]}_2,\\
 m &= 4\left(\norm{\norm{A(\omega)}_2 }_{\psi_{1}}+\sup_{x,z\in\mathcal{C}}\norm{K(x,z)}\right),\enskip \omega \sim \mu,\\
\sigma_p^2 &= \expectation_{\mu}\left[\norm{\omega}_2^2\norm{A(\omega)}_2^2\right].
\end{align*}
Then for all $\epsilon$ in $\mathbb{R}_+$,
\begin{multline*}
\probability \left\{ \norm{F}_\infty \ge \epsilon \right\} \le C_d \left( \frac{\sigma_p l}{\epsilon} \right)^{\frac{2}{1+2/d}} \begin{cases}
\exp\left(-\frac{\epsilon^2D}{8(d+2)\left(b_D+\frac{\epsilon\bar{u}_D}{6}\right)} \right) & \text{if}\enskip \bar{u}_D\le \frac{2(e-1)b_D}{\epsilon} \\
\exp\left( -\frac{\epsilon D}{(d+2)(e-1)\bar{u}_D} \right) & \text{otherwise},
\end{cases}
\end{multline*}
where
$\bar{u}_D = 2m\log\left( 2^{\frac{3}{2}} \left(\frac{m}{b_D}\right)^2\right)$ and $
C_d = p\left(\left(\frac{d}{2}\right)^{\frac{-d}{d+2}}+\left(\frac{d}{2}\right)^{\frac{2}{d+2}}\right)2^{\frac{6d+2}{d+2}}$.
\end{theorem}
We detail the proof of the theorem in \cref{subsec:concentration_proof}. It follows the usual scheme derived in \citet{Rahimi2007} and \citet{sutherland2015} and involves Bernstein concentration inequality for unbounded symmetric matrices (\cref{th:bernstein1}).
\subsection{Application to some operator-valued kernel}
To apply \cref{pr:ov-rff_concentration} to operator-valued kernels, we need to ensure that all the constants exist. In the following, we first show how to bound the constant term $b_D$. Then we exhibit the upper bounds for the three operator-valued kernels we took as examples. Even\-tually, we ensure that the random variable $\norm{A(\omega)}$ has a finite Orlicz norm with $\psi=\psi_1$ in these three cases.
\paragraph{Bounding the term $b_D(\delta)$:}
\begin{proposition}\label{pr:variance-bound} Define the matrix $\variance_{\mu}[A(\omega)]$ as follows: for all $\ell,m\inrange{1}{p}$,
\begin{equation*}
 \variance_{\mu}[A(\omega)]_{\ell m} = \sum_{r=1}^p \covariance{A(\omega)_{\ell r}, A(\omega)_{r m}}
\end{equation*}
For a given $\delta=x-z$, define:
\begin{equation*}
  b_D(\delta)= D\norm{\expectation_{\mu}\left[\sum_{j=1}^D \left(F_j(\delta)\right)^2\right]}_2.
\end{equation*}
Then we have:
\begin{equation*}
  b_D(\delta) \le \frac{1}{2} \norm{(K_0(2
\delta)+K_0(0))\expectation_\mu[A(\omega)] - 2K_0(\delta)^2}_2 + \norm{\variance_{\mu}[A(\omega)]}_{2}.
\end{equation*}
\end{proposition}
The proof uses trigonometry properties and various properties of the moments and is given in \cref{sec:bd_application}. Now, we compute the upper bound given by \cref{pr:variance-bound} for the three kernels we have taken as examples.
\setlist[enumerate,1]{leftmargin=0.98cm}
\begin{enumerate}[i)]
  \item \emph{Decomposable kernel:} notice that in the case of the Gaussian decomposable kernel, i.e. $A(\omega)=A$, $K_0(\delta)= Ak_0(\delta)$, $k_0(\delta) \geq 0$ and $k_0(\delta)=1$, then we have:
\begin{equation*}
b_D(\delta) \leq \frac{1}{2}(1+k_0(2\delta))\norm{A}_2 + k_0(\delta)^2
\end{equation*}
  \item \emph{curl-free and div-free kernels:} recall that in this case $p=d$. For the (Gaussian) curl-free kernel, $A(\omega)=\omega\omega^*$ where $\omega\in\mathbb{R}^d\sim\mathcal{N}(0, \sigma^{-2}I_d)$ thus $\expectation_\mu [A(\omega)] = I_d/\sigma^2$ and $\variance_{\mu}[A(\omega)]=(d+1)I_d/\sigma^4$ (see \citet{petersen2008matrix}). Hence,
\begin{equation*}
b_D(\delta) \leq \frac{1}{2}\norm{\frac{1}{\sigma^2}K_0(2\delta)-2 K_0(\delta)^2}_2 + \frac{(d+1)}{\sigma^4}
\end{equation*}
Eventually for the Gaussian divergence-free kernel, $A(\omega)=I\norm{\omega}_2^2-\omega\omega^*$, thus $\expectation_\mu [A(\omega)] = I_d(d-1)/\sigma^2$ and $ \variance_{\mu}[A(\omega)]=d(4d-3)I_d/\sigma^4$ (see \citet{petersen2008matrix}). Hence,
\begin{equation*}
b_D(\delta) \leq \frac{1}{2}\norm{\frac{(d-1)}{\sigma^2}K_0(2\delta)-2K_0(\delta)^2}_2+ \frac{d(4d-3)}{\sigma^4}
\end{equation*}
\end{enumerate}
An empirical illustration of these bounds is shown in \cref{fig:variance_error}.
\paragraph{Computing the Orlicz norm:}
For a random variable with strictly monotonic moment generating function (MGF),
one can characterize its $\psi_1$ Orlicz norm by taking the functional inverse of the MGF evaluated at 2.
In other words
\begin{equation*}
 \norm{X}_{\psi_1}^{-1}=\MGF(x)^{-1}_X(2).
\end{equation*}
For the Gaussian curl-free and divergence-free kernel \begin{equation*}
    \norm{A^{div}(\omega)}_2=\norm{A^{curl}(\omega)}_2=\norm{\omega}_2^2
\end{equation*}
where $\omega\sim\mathcal{N}(0,I_d/\sigma^2)$, hence $\norm{A(\omega)}_2\sim \Gamma(p/2,2/\sigma^2)$. The MGF of this gamma distribution is $\MGF(x)^{-1}(t)=(1-2t/\sigma^2)^{-(p/2)}$. Eventually
\begin{equation*}
 \norm{\norm{A^{div}(\omega)}_2}_{\psi_1}^{-1}=\norm{\norm{A^{curl}(\omega)}_2}_{\psi_1}^{-1}=\frac{\sigma^2}{2}\left(1-4^{-\frac{1}{p}}\right).
\end{equation*}

\section{Learning with ORFF}
\label{sec:learning_with_ORFF}
In practise, the previous bounds are however too large to find a safe value for $D$. In the following, numerical examples of ORFF-approximations are presented.
\subsection{Penalized regression with ORFF}
Once we have an approximated feature map, we can use it to provide a feature matrix of size $p'D \times p$ with matrix $B(\omega)$ of size $p \times p'$ such that $A(\omega)=B(\omega)B(\omega)^*$. A function $f \in \bmH_K$ is then approximated by a linear model
\begin{equation*}
 \tilde{f}(x)=\tilde{\Phi}(x)^*\theta,\enskip\text{where}\enskip\theta \in \mathbb{R}^{p'D}.
\end{equation*}
Let $\bmS = \{(x_i,y_i) \in \mathbb{R}^d \times \mathbb{R}^p, i=1, \ldots, N\}$ be a collection of i.i.d training samples.
Given a local loss function $L: \bmS\to \mathbb{R}^+$ and a $\ell_2$ penalty, we minimize
\begin{equation}
 \label{eq:opt_ridge}
 \bmL(\theta)=\frac{1}{N}\sum_{i=1}^N L\left(\tilde{\Phi}(x_i)^*\theta,y_i\right)+\lambda\norm{\theta}_2^2,
\end{equation}
instead of minimizing $\bmL(f)=\frac{1}{N}\sum_{i=1}^N L(f(x_i),y_i)+\lambda\norm{f}^2_{\mathcal{H}_K}$. To find a minimizer of the optimization problem \cref{eq:opt_ridge} many optimization algorithms are available. For instance, in large-scale context, a stochastic gradient descent algorithm would be be suitable: we can adapt the algorithm to the kind of kernel/problematic. We investigate two optimization algorithms: a Stein equation solver appropriate for the decomposable kernel and a (stochastic) gradient descent for non-decomposable kernels (e.g. the curl-free and divergence-free kernels).

\paragraph{Closed form for the decomposable kernel:}
for the real decomposable kernel $K_0(\delta)=k(\delta)A$ when $L(y,y')=\norm{y-y'}_2^2$ (Kernel Ridge regression in $\bmH_K$), the learning problem described in \cref{eq:opt_ridge} can be re-written in terms of matrices to find the unique minimizer $\Theta_*$, where $\text{vec}(\Theta)=\theta$ such that $\theta$ is a $p'D$ vector and $\Theta$ a $p'\times D$ matrix.
If $\tilde{\phi}$ is a feature map ($\tilde{\phi}(X)$ is a matrix of size $D\times N$) for the scalar kernel $k_0$, then
\begin{equation*}
 \tilde{\Phi}(x)^*\theta=(\tilde{\phi}(x)^*\otimes B)\theta=B\Theta\tilde{\phi}(x)
\end{equation*}
and
\begin{equation}
 \label{eq:stein}
 \theta_* = \argmin_{\Theta\in\mathbb{R}^{p'\times D}}\norm{B\Theta \tilde{\phi}(X) -Y}_F^2+\lambda\norm{\Theta}_F^2.
\end{equation}
This is a convex optimization problem and a sufficient condition is:
\begin{equation*}
\tilde{\phi}(X)\tilde{\phi}(X)^*\Theta_* B^*B - \tilde{\phi}(X)Y^*B + \lambda\Theta_* = 0,
\end{equation*}
which is a Stein equation.
\paragraph{Gradient computation for the general case.}
When it is not possible or desirable to use Stein's equations solver one can apply a (stochastic) gradient descent algorithm. The gradient computation for and $\ell_2$-loss applied to ORFF model is briefly recalled in \cref{subsec:implementation_detail}.
\subsection{Numerical illustration}
We present a few experiments to complete the theoretical contribution and illustrate the behavior of ORFF-regression. Other experimentalresults with noisy output data are shown in \cref{sec:more_simulation}.
\begin{figure}
 \centering
 \includegraphics[width=0.9\textwidth]{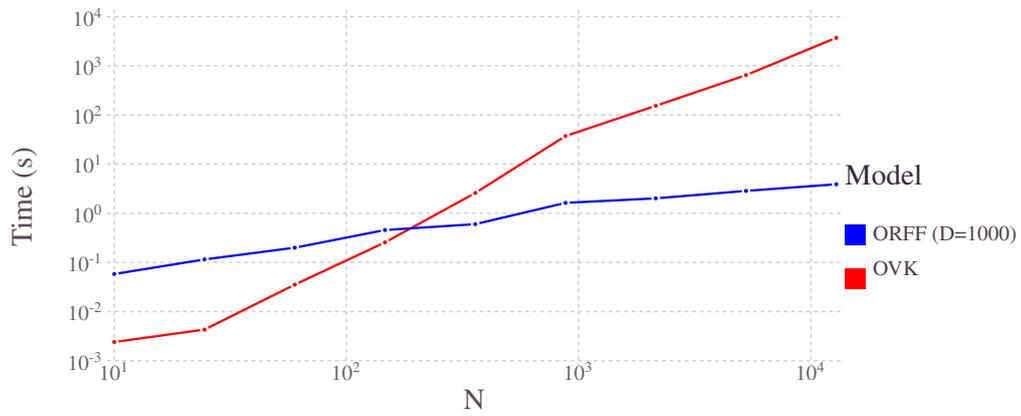}
 \caption[Computation time for ORFF and OVK on MNIST versus the number of datapoints $N$]{Computation time of ORFF and OVK on MNIST versus the number of datapoints $N$.}
 \label{fig:learning_time}
\end{figure}
\begin{figure}
 \centering
 \includegraphics[width=0.9\textwidth]{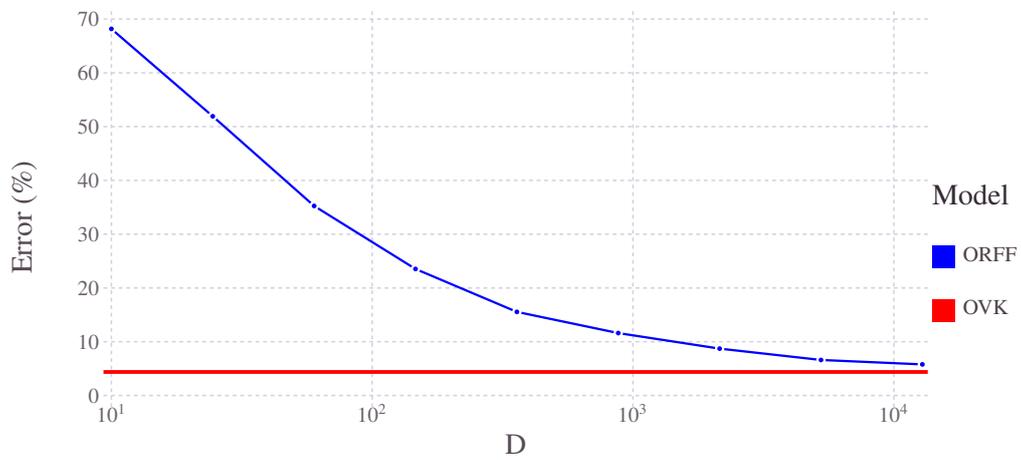}
 \caption[Prediction Error in percent on the MNIST dataset versus $D$, the number of Fourier features]{Prediction Error in percent on MNIST versus $D$, the number of Fourier features. In blue dashed line, ORFF and in red solid line OVK. For OVK and ORFF the number of datapoints $N=1000$.}
 \label{fig:learning_accuracy}
\end{figure}
\paragraph{Datasets:}
the first dataset is the handwritten digits recognition dataset MNIST\footnote{available at \url{http://yann.lecun.com/exdb/mnist}.}
We select a training set of $12000$ images and a test set of $10000$ images. The inputs are images represented as a vector $x_i\in[0,255]^{784}$ and the targets are integers between $0$ and $9$. First we scaled the inputs such that they take values in $[-1,1]^{784}$. Then we binarize the targets such that each number is represented by a unique binary vector of length $10$. To predict classes, we use simplex coding method presented in \citet{mroueh2012multiclass}. The intuition behind simplex coding is to project the binarized labels of dimension $p$ onto the most separated vectors on the hypersphere of dimension $p-1$. For ORFF we can encode directly this projection in the $B$ matrix of the decomposable kernel $K_0(\delta)=B B^* k_0(\delta)$ where $k_0$ is a Gaussian kernel. For OVK we project the binarized targets on the simplex as a preprocessing step, before learning with the kernel $K_0(\delta)=I_p k_0(\delta)$, where $k_0$ is a also Gaussian kernel.
\par
The second dataset corresponds to a 2D-vector field with structure. We generated a scalar field as a mixture of five Gaussians located at $[0,0]$, $[0,1]$, $[0,-1]$, with positive values and at $[-1,0]$, $[1, 0]$ with negative values. The curl-free field has been generated by taking the gradient of the scalar-field, and the divergence-free field by taking the orthogonal of the curl-free field. These 2D-datasets are depicted in \cref{fig:curl-div-fields}.
\paragraph{Approximation:} We trained both an ORFF and an OVK model on the handwritten digits recognition dataset (MNIST) with a decomposable Gaussian kernel with signature $K_0(\delta)=\exp(-\norm{\delta}/\sigma^2)A$.
To find a solution of the optimization problem described in \cref{eq:stein}, we use off-the-shelf solver\footnote{ \url{http://ta.twi.tudelft.nl/nw/users/gijzen/IDR.html}} able to handle Stein's equation. For both methods we choose $\sigma=20$ and use a $2$-fold cross validation on the training set to select the optimal $\lambda$. First, \cref{fig:learning_time} shows the running time comparison between OVK and ORFF models using $D=1000$ Fourier features against the number of data\-points $N$. The log-log plot shows ORFF scaling better than the OVK w.r.t the number of points.
Second, \cref{fig:learning_accuracy} shows the test prediction error versus the number of ORFFs $D$, when using $N=1000$ training points. As expected, the ORFF model converges toward the OVK model when the number of features increases.
\paragraph{Independent (RFF) prediction vs Structured prediction on vector fields:}
we perform a similar experiment over a simulated dataset designed for learning a 2D-vector field with structure. \Cref{fig:curl-div} reports the Mean Squared Error versus the number of ORFF $D$. For this experiment we use a Gaussian curl-free kernel and tune its $\sigma$ hyperparameter as well as the $\lambda$ on a grid. The curl-free ORFF outperforms classic RFFs by tending more quickly towards the noise level.
\begin{figure}
 \centering
 \includegraphics[width=0.9\textwidth]{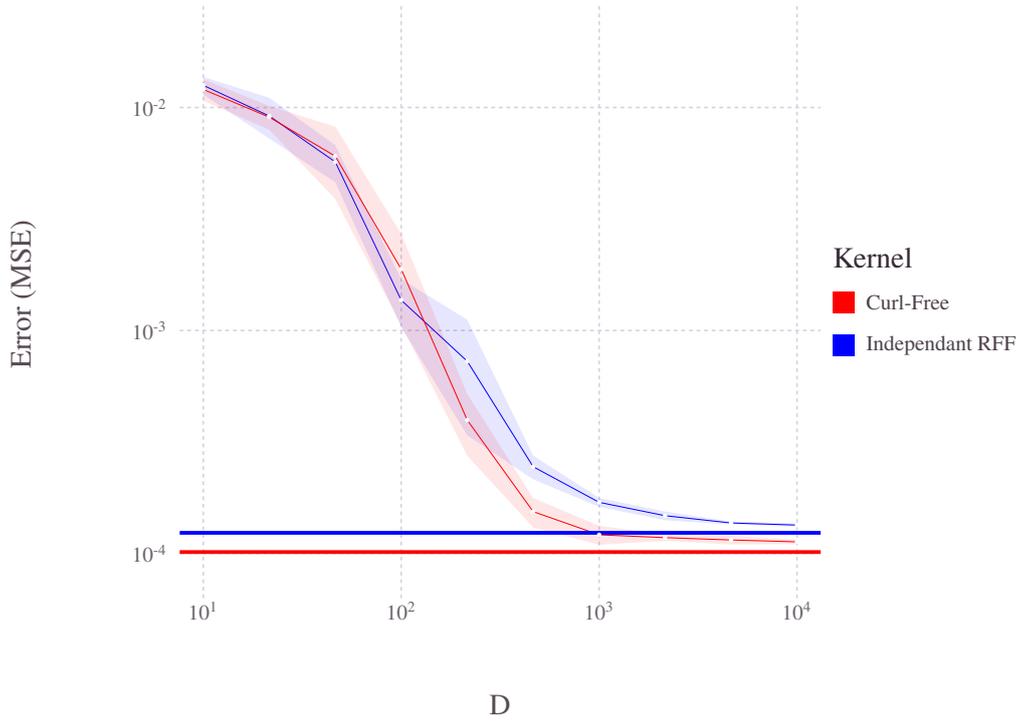}
 \caption[Mean squared (test) error on the synthetic data versus number of Fourier features $D$]{Mean squared (test) error on the synthetic data versus number of Fourier features $D$. The solid lines represent decomposable (blue) and curl (red) OVK methods while the dotted lines represent decomposable (blue) and curl (red) ORFF methods.}
 \label{fig:curl-div}
\end{figure}
\Cref{fig:learning_time_curl} shows the computation time between curl-ORFF and curl-OVK indicating that the OVK solution does not scale to large datasets, while ORFF scales well with when the number of data increases. When $N>10^4$ exact OVK is not able to be trained in reasonable time ($>1000$ seconds).
\begin{figure}
 \centering
 \includegraphics[width=0.9\textwidth]{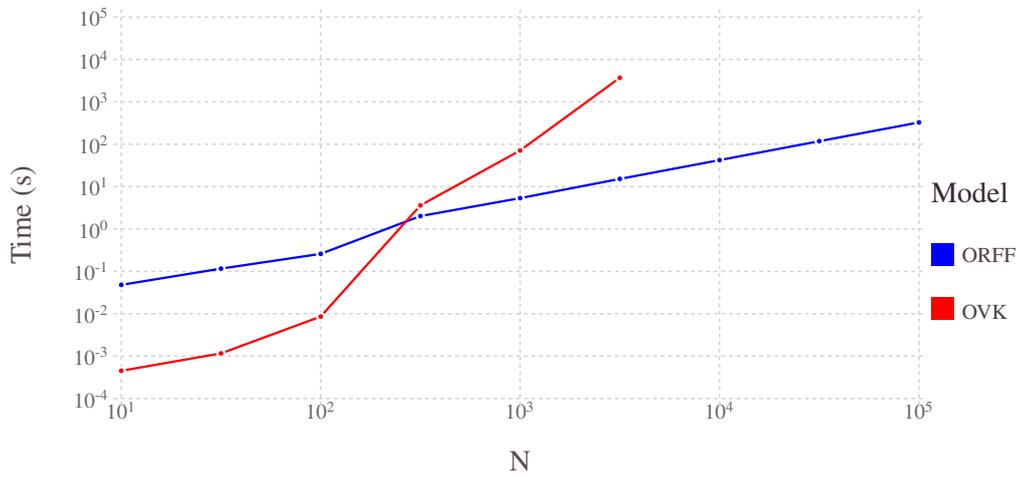}
 \caption[Computation time of curl-ORFF and curl-OVK versus the number of datapoints on synthetic data]{Computation time of curl-ORFF and curl-OVK versus the number of datapoints on synthetic data. We fixed $D=1000$ Fourier features and study the computation time w.r.t the number of points.}
 \label{fig:learning_time_curl}
\end{figure}
\section{Conclusion}
\label{sec:conclusion}
We introduced a general and versatile framework for operator-valued kernel approximation with Operator Random Fourier Features. We showed the uniform convergence of these approximations by proving a matrix concentration inequality for bounded and unbounded ORFFs. The complexity in time of these approximations together with the linear learning algorithm make this implementation scalable with the number of data and therefore interesting compared to OVK regression. The numerical illustration shows the behavior expected from theory. ORFFs are especially a very promising approach in vector field learning or on noisy datasets.
 Another appealing direction is to use this architecture to automatically learn operator-valued kernels by learning a mixture of ORFFs in order to choose appropriate kernels, a working direction closely related to the recent method called ``Alacarte'' \citep{Yang2015} based on the very efficient ``FastFood'' method \citep{Le2013} for scalar kernels. Finally this work opens the door to building deeper architectures by stacking vector-valued functions while keeping a kernel view for large datasets.

\bibliography{main_axiv.bib}
\bibliographystyle{abbrvnat}

\cleardoublepage
\newgeometry{
 twoside,
 top=20mm,
 inner=25mm,
 outer=25mm,
 bottom=40mm,
 headheight=15pt,
 headsep=5mm,
 footnotesep=5mm,
 footskip=20mm - 1em}

\footnotesize
\appendix
\label{appendix}

\section{Reminder on Random Fourier Feature in the scalar case}
\cite{Rahimi2007} proved the uniform convergence of Random Fourier Feature (RFF) approximation for a scalar shift invariant kernel.
\begin{theorem}[Uniform error bound for RFF, \citet{Rahimi2007}]\label{rff-scalar-bound}

Let $\mathcal{C}$ be a compact of subset of $\mathbb{R}^d$ of diameter $l$. Let $k$ a shift invariant kernel, differentiable with a bounded first derivative and $\mu$ its normalized inverse Fourier transform. Let $D$ the dimension of the Fourier feature vectors. Then, for the mapping $\tilde{\phi}$ described in \cref{sec:orff}, we have :
\begin{equation}
\probability\left\{\sup_{x,z\in\mathcal{C}}\norm{\tilde{k}(x,z)-k(x,z)}_2\ge \epsilon \right\} \le 2^8\left( \frac{d\sigma l}{\epsilon} \right)^2\exp\left( -\frac{\epsilon^2D}{4(d+2)} \right)
\end{equation}
\end{theorem}
From \cref{rff-scalar-bound}, we can deduce the following corollary about the uniform convergence of the ORFF approximation of the decomposable kernel.

\begin{corollary}[Uniform error bound for decomposable ORFF]\label{sec:dec-bound}
Let $\mathcal{C}$ be a compact of subset of $\mathbb{R}^d$ of diameter $l$. $K_{dec}$ is a decomposable kernel built from a $p \times p$ semi-definite matrix $A$ and $k$, a shift invariant and differentiable kernel whose first derivative is bounded. Let $\tilde{k}$ the Random Fourier approximation for the scalar-valued kernel $k$. We recall that: for a given pair $(x,z) \in \bmC$, $\tilde{K}(x,z)= \tilde{\Phi}(x)^* \tilde{\Phi}(z)=A\tilde{k}(x,z)$ and $K_0(x-z)=A E_{\mu}[\tilde{k}(x,z)]$.
\begin{equation*}
\probability\left\{\sup_{x,z\in\mathcal{C}}\norm{\tilde{K}(x,z)-K(x,z)}_2\ge \epsilon \right\} \le 2^8\left( \frac{d\sigma \norm{A}_2 l}{\epsilon} \right)^2\exp\left( -\frac{\epsilon^2D}{4\norm{A}_2^2(d+2)} \right)
\end{equation*}
\begin{proof}
The proof directly extends \ref{rff-scalar-bound} given by \cite{Rahimi2007}. Since
\begin{equation*}
    \sup_{x,z\in\mathcal{C}}\norm{\tilde{K}(x,z)-K(x,z)}_2 = \sup_{x,z\in\mathcal{C}} \norm{A}_2. \abs{\tilde{K}(x,z)-k(x,z)}
\end{equation*}
and then, taking $\epsilon' = \norm{A}_2 \epsilon$ gives the following result for all positive $\epsilon'$:
\begin{equation*}
\probability\left\{\sup_{x,z\in\mathcal{C}}\norm{A(\tilde{K}(x,z)-k(x,z))}_2\ge \epsilon' \right\} \le 2^8\left( \frac{d\sigma \norm{A}_2 l}{\epsilon'} \right)^2\exp\left( -\frac{\epsilon'^2D}{4\norm{A}_2^2(d+2)} \right)
\end{equation*}
\end{proof}
\end{corollary}
Please note that a similar corollary could have been obtained for the recent result of \citet{sutherland2015} who refined the bound proposed by Rahimi and Recht by using a Bernstein concentration inequality instead of the Hoeffding inequality.

\section{Proof of the uniform error bound for ORFF approximation}
This section present a proof of \cref{pr:ov-rff_concentration}.
\label{subsec:concentration_proof}
\paragraph{}
We note $\delta=x-z$, $\tilde{K}(x,z)=\tilde{\Phi}(x)^*\tilde{\Phi}(z)$, $\tilde{K}_j(x,z)=\tilde{\Phi}_j(x)^*\tilde{\Phi}_j(z)$ and $K_0(\delta)=K(x,z)$.
For sake of simplicity, we use the following notation:
\begin{eqnarray*}
F(\delta)&=&F(x-z)=\tilde{K}(x,z)-K(x,z)\\
F_j(\delta)&=&F_j(x-z)=(\tilde{K}_j(x,z)-K(x,z))/D
\end{eqnarray*}

Compared to the scalar case, the proof follows the same scheme as the one described in \citep{Rahimi2007, sutherland2015} but requires to consider matrix norms and appropriate matrix concentration inequality. The main feature of \cref{pr:ov-rff_concentration} is that it covers the case of bounded ORFF as well as unbounded ORFF: in the case of bounded ORFF, a Bernstein inequality for matrix concentration such that the one proved in \cite{Mackey2014} (Corollary 5.2) or the formulation of \cite{Tropp} recalled in \cite{Koltchinskii2013remark} is suitable. However some kernels like the curl and the div-free kernels do not have bounded $\norm{F_j}$ but exhibit $F_j$ with subexponential tails. Therefore, we will use a Bernstein matrix concentration inequality adapted for random matrices with subexponential norms (\cite{Koltchinskii2013remark}).


\subsection{Epsilon-net}
Let $\mathcal{D}_{\bmC}=\left\{ x-z \quad\middle|\quad x,z\in\mathcal{C} \right\}$ with diameter at most $2l$ where $l$ is the diameter of $\mathcal{C}$. Since $\mathcal{C}$ is supposed compact, so is $\mathcal{D}_{\bmC}$. It is then possible to find an $\epsilon$-net covering $\mathcal{D}_{\bmC}$ with at most $T=(4l/r)^d$ balls of radius $r$.
\par
Let us call $\delta_i,i=1,\ldots,T$ the center of the $i$-th ball, also called anchor of the $\epsilon$-net. Denote $L_{F}$ the Lipschitz constant of $F$. Let $\norm{.}$ be the $\ell_2$ norm on $\bmL(\mathbb{R}^p)$, that is the spectral norm. Now let use introduce the following lemma:
\begin{lemma}
$\forall \delta \in \bmD_{\bmC}$, if (1): $L_{F}\le\frac{\epsilon}{2r}$ and (2): $\norm{F(\delta_i)}\le\frac{\epsilon}{2}$,for all $0<i<T$, then $\norm{F(\delta)} \leq \epsilon$.
\end{lemma}
\begin{proof}
$\norm{F(\delta)}=\norm{F(\delta)-F(\delta_i)+F(\delta_i)}\le\norm{F(\delta)-F(\delta_i)}+\norm{F(\delta_i)}$, for all $0<i<T$. Using the Lipschitz continuity of $F$ we have $\norm{F(\delta)-F(\delta_i)}\le L_{F}\norm{\delta-\delta_i}\le rL_{F}$ hence
$\norm{F(\delta)} \le rL_{F} + \norm{F(\delta_i)}$.
\end{proof}
To apply the lemma, we must bound the Lipschitz constant of the matrix-valued function $F$ (condition (1)) and $\norm{F(\delta_i)}$, for all $i=1, \ldots, T$ as well (condition (2)).
\subsection{Regularity condition}
\label{subsubsec:regularity}
We first establish that $\frac{\partial}{\partial \delta}\expectation \tilde{K}(\delta) = \expectation \frac{\partial}{\partial \delta}\tilde{K}(\delta)$. Since $\tilde{K}$ is a finite dimensional matrix-valued function, we verify the integrability coefficient-wise, following \citet{sutherland2015}'s demonstration. Namely, without loss of generality we show
\begin{equation*}
  \left[\frac{\partial}{\partial \delta}\expectation \tilde{K}(\delta)\right]_{lm} = \expectation \frac{\partial}{\partial \delta} \left[\tilde{K}(\delta)\right]_{lm}
\end{equation*}
where $[A]_{lm}$ denotes the $l$-th row and $m$-th column element of the matrix A.
\begin{proposition}[Differentiation under the integral sign]
\label{pr:diff_under_int}
Let $\mathcal{X}$ be an open subset of $\mathbb{R}^d$ and $\Omega$ be a measured space. Suppose that the function $f:\mathcal{X}\times\Omega\to\mathbb{R}$ verifies the following conditions:
\begin{itemize}
  \item $f(x,\omega)$ is a measurable function of $\omega$ for each $x$ in $\mathcal{X}$.
  \item For almost all $\omega$ in $\Omega$, the derivative $\partial f(x, \omega)/\partial x_i$ exists for all $x$ in $\mathcal{X}$.
  \item There is an integrable function $\Theta:\Omega\to\mathbb{R}$ such that $\abs{\partial f(x, \omega)/\partial x_i}\le\Theta(\omega)$ for all $x$ in $\mathcal{X}$.
\end{itemize}
Then
\begin{equation*}
  \frac{\partial}{\partial x_i} \int_\Omega f(x,\omega)d\omega = \int_\Omega \frac{\partial}{\partial x_i}f(x,\omega)d\omega.
\end{equation*}
\end{proposition}
Define the function $\tilde{G}_{x,y}^{i,l,m}(t,\omega):\mathbb{R}\times\Omega\to\mathbb{R}$ by $\tilde{G}_{x,y}^{i,l,m}(t,\omega)=\left[\tilde{K}(x+te_i-y)\right]_{lm}=\left[\tilde{G}_{x,y}^{i}(t,\omega)\right]_{lm}$, where $e_i$ is the $i$-th standard basis vector.
Then $\tilde{G}_{x,y}^{i,l,m}$ is integrable w.r.t. $\omega$ since
\begin{equation*}
  \int_\Omega \tilde{G}_{x,y}^{i,l,m}(t,\omega) d\omega= \expectation \left[\tilde{K}(x+te_i-y)\right]_{lm}=\left[K(x+te_i-y)\right]_{lm} < \infty.
\end{equation*}
Additionally for any $\omega$ in $\Omega$, $\partial/\partial t \tilde{G}_{x,y}^{i,l,m}(t,\omega)$ exists and satisfies
\begin{equation*}\scriptsize
  \begin{aligned}
    \expectation \abs{\frac{\partial}{\partial t}\tilde{G}_{x,y}^{i,l,m}(t,\omega)}
    &= \expectation \abs{\frac{1}{D}\sum_{j=1}^DA(\omega)_{lm}\left(\sin\inner{y,\omega_j}\frac{\partial}{\partial t}\sin(\inner{x,\omega_j}+t\omega_{ij})+\cos\inner{y,\omega_j}\frac{\partial}{\partial t}\cos(\inner{x,\omega_j}+t\omega_{ij})\right)} \\
    &=\expectation \abs{\frac{1}{D}\sum_{j=1}^DA(\omega)_{lm}\left(\omega_{ji}\sin\inner{y,\omega_j}\sin(\inner{x,\omega_j}+t\omega_{ji})-\omega_{ji}\cos\inner{y,\omega_j}\cos(\inner{x,\omega_j}+t\omega_{ji})\right)} \\
    &\le\expectation\left[ \frac{1}{D}\sum_{j=1}^D\abs{A(\omega)_{lm}\omega_{ji}\sin\inner{y,\omega_j}\sin(\inner{x,\omega_j}+t\omega_{ji})}+\abs{A(\omega)_{lm}\omega_{ji}\cos\inner{y,\omega_j}\cos(\inner{x,\omega_j}+t\omega_{ji})}\right] \\
    &\le \expectation \left[\frac{1}{D}\sum_{j=1}^D2\abs{A(\omega)_{lm}\omega_{ji}}\right].
  \end{aligned}
\end{equation*}
Hence
\begin{equation*}
  \begin{aligned}
  \expectation \abs{\frac{\partial}{\partial t}\tilde{G}_{x,y}^{i}(t,\omega)} &\le 2\expectation \left[\abs{\omega \otimes A(\omega)}\right].
  \end{aligned}
\end{equation*}
which is assumed to exist since in finite dimensions all norms are equivalent and $\expectation_{\mu}\left[\norm{\omega}^2\norm{A(\omega)}^2\right]$ is assume to exists.
Thus applying \cref{pr:diff_under_int} we have $\left[\frac{\partial}{\partial \delta_i}\expectation \tilde{K}(\delta)\right]_{lm} = \expectation \frac{\partial}{\partial \delta_i} \left[\tilde{K}(\delta)\right]_{lm}$ The same holds for $y$ by symmetry.
Combining the results for each component $x_i$ and for each element $lm$, we get that $\frac{\partial}{\partial \delta}\expectation \tilde{K}(\delta) = \expectation \frac{\partial}{\partial \delta}\tilde{K}(\delta)$.

\subsection{Bounding the Lipschitz constant}
\paragraph{}
Since $F$ is differentiable, $L_{F}=\norm{\frac{\partial F}{\partial \delta}(\delta^*)}$ where $\delta^*=\argmax_{\delta\in\bmD_{\bmC}}\norm{\frac{\partial F}{\partial \delta}(\delta)}$.
\begin{equation*}
\begin{aligned}
\expectation_{\mu,\delta^*}\left[ L_f^2 \right] &= \expectation_{\mu,\delta^*} \norm{\frac{\partial \tilde{K}}{\partial \delta}(\delta^*)-\frac{\partial K_0}{\partial \delta}(\delta^*)}^2 \\
&\le \expectation_{\delta^*}\left[ \expectation_{\mu} \norm{\frac{\partial \tilde{K}}{\partial \delta}(\delta^*)}^2 - 2\norm{\frac{\partial K_0}{\partial \delta}(\delta^*)}\expectation_{\mu} \norm{\frac{\partial \tilde{K}}{\partial \delta}(\delta^*)} + \norm{\frac{\partial K_0}{\partial \delta}(\delta^*)}^2 \right]
\end{aligned}
\end{equation*}


Using Jensen's inequality $ \norm{\expectation_\mu\frac{\partial \tilde{K}}{\partial \delta}(\delta^*)}\le\expectation_\mu\norm{\frac{\partial \tilde{K}}{\partial \delta}(\delta^*)}$ and $\frac{\partial}{\partial \delta}\expectation \tilde{K}(\delta) = \expectation \frac{\partial}{\partial \delta}\tilde{K}(\delta)$. Since $\tilde{K}$ (see \cref{subsubsec:regularity}), $\expectation_\mu\frac{\partial \tilde{K}}{\partial \delta}(\delta^*)=\frac{\partial}{\partial \delta}\expectation_\mu\tilde{K}(\delta^*)=\frac{\partial K_0}{\partial \delta}(\delta^*)$ thus
\begin{equation*}
\begin{aligned}
\expectation_{\mu,\delta^*}\left[ L_f^2 \right] &\le \expectation_{\delta^*}\left[ \expectation_\mu \norm{\frac{\partial \tilde{K}}{\partial \delta}(\delta^*)}^2 - 2\norm{\frac{\partial K_0}{\partial \delta}(\delta^*)}^2 + \norm{\frac{\partial K_0}{\partial \delta}(\delta^*)}^2 \right] \\
&= \expectation_{\mu,\delta^*}\norm{\frac{\partial \tilde{K}}{\partial \delta}(\delta^*)}^2-\expectation_{\delta^*}\norm{\frac{\partial K_0}{\partial \delta}(\delta^*)}^2 \\
&\le \expectation_{\mu,\delta^*}\norm{\frac{\partial \tilde{K}}{\partial \delta}(\delta^*)}^2 \\
&= \expectation_{\mu,\delta^*}\norm{\frac{\partial }{\partial \delta^*}\cos\inner{\delta^*,\omega}A(\omega) }^2 \\
&= \expectation_{\mu,\delta^*}\norm{ -\omega\sin( \inner{\delta^*,\omega}) \otimes A(\omega) }^2 \\
&\le \expectation_{\mu}\left[\norm{\omega}^2\norm{A(\omega)}^2\right] \triangleq \sigma_p^2
\end{aligned}
\end{equation*}
Eventually applying Markov's inequality yields
\begin{equation}
\label{eq:Lipschitz_bound}
\probability \left\{ L_{F} \ge \frac{\epsilon}{2r} \right\} = \probability \left\{ L_{F}^2 \ge \left(\frac{\epsilon}{2r}\right)^2 \right\} \le \sigma_p^2\left(\frac{2r}{\epsilon} \right)^2.
\end{equation}

\subsection[Bounding F on a given anchor point]{Bounding $F$ on a given anchor point $\delta_i$}

To bound $\norm{F(\bfdelta_i)}_2$, Hoeffding inequality devoted to matrix concentration \cite{Mackey2014} can be applied. We prefer here to turn to tighter and refined inequalities such as Matrix Bernstein inequalities (\citet{sutherland2015} also pointed that for the scalar case).
\par
If we had bounded ORFF, we could use the following Bernstein matrix concentration inequality proposed in \cite{Ahls2002,Tropp,Koltchinskii2013remark}.

\begin{theorem}[Bounded non-commutative Bernstein concentration inequality]\label{th:bernstein1}
Verbatim from Theorem 3 of \citet{Koltchinskii2013remark},
consider a sequence $(X_{j})_{j=1}^D$ of $D$ independent Hermitian (here symmetric) $p \times p$ random matrices that satisfy $\expectation X_{j} = 0$ and suppose that for some constant $U > 0$, $\norm{X_{j}} \leq U$ for each index $j$. Denote $B_D = \norm{\expectation[X_1^2 + \ldots X_D^2]}$.
Then, for all $\epsilon \geq 0$,
\begin{equation*}
\probability\left\{\norm{\sum_{j=1}^D X_{j}} \geq \epsilon \right\} \leq p \exp\left(-\frac{\epsilon^2}{2B_D + 2U\epsilon/3}\right)
\end{equation*}
\end{theorem}
However, to cover the general case including unbounded ORFFs like curl and div-free ORFFs, we choose a version of Bernstein matrix concentration inequality proposed in \cite{Koltchinskii2013remark} that allow to consider matrices are not uniformly bounded but have subexponential tails.
\begin{theorem}[Unbounded non-commutative Bernstein concentration inequality]
\label{th:matrix_bernstein}
Verbatim from Theorem 4 of \citet{Koltchinskii2013remark}.
Let $X_1,\hdots,X_D$ be independent Hermitian $p\times p$ random matrices, such that $\expectation X_j = 0$ for all $j=1,\hdots,D$. Let $\psi=\psi_1$. Define
\begin{equation*}
F_{(D)} \triangleq \sum_{j=1}^D X_j \quad \text{and} \quad B_D \triangleq \norm{\expectation\left[ \sum_{j=1}^D X_j^2 \right]}.
\end{equation*}
Suppose that,
\begin{equation*}
M = 2\max_{1\le j\le D}\norm{\norm{X_j}}_{\psi}
\end{equation*}
Let $\delta\in\left]0;\frac{2}{e-1}\right[$ and
\begin{equation*}
\bar{U}\triangleq M\log \left( \frac{2}{\delta}\frac{M^2}{B_D^2}+1\right)
\end{equation*}
Then, for $\epsilon\bar{U}\le(e-1)(1+\delta)B_D$,
\begin{equation}
\probability\left\{ \norm{F_{(D)}} \ge \epsilon \right\} \le 2p\exp\left( -\frac{\epsilon^2}{2(1+\delta)B_D+2\epsilon\bar{U}/3}\right)
\end{equation}
and for $\epsilon\bar{U}>(e-1)(1+\delta)B_D$,
\begin{equation}
\probability\left\{ \norm{F_{(D)}} \ge \epsilon \right\} \le 2p\exp\left( -\frac{\epsilon}{(e-1)\bar{U}} \right).
\end{equation}
\end{theorem}

To use this theorem, we set: $X_j=F_j(\delta_i)$.
We have indeed: $\expectation_{\mu}[F_j(\delta_i)] = 0$ since $\tilde{K}(\delta_i)$ is the Monte-Carlo approximation of $K_0(\delta_i)$ and the matrices $F_j(\delta_i)$ are symmetric.
We assume we can bound all the Orlicz norms of the $F_j(\delta_i)=\frac{1}{D}(\tilde{K}_j(\delta_i) - K_0(\delta_i))$. Please note that in the following we use a constant $m$ such that $m=D M$,
\begin{equation*}
\begin{aligned}
m&=2D\max_{1\le j\le D}\norm{\norm{F_j(\delta_i)}}_{\psi}\\
 &\le 2\max_{1\le j\le D}\norm{\norm{\tilde{K}_j(\delta_i)}}_{\psi} + 2\norm{\norm{K_0(\delta_i)}}_{\psi}\\
 &<4\max_{1\le j\le D}\norm{\norm{A(\omega_j)}}_{\psi}+4\norm{K_0(\delta_i)}
\end{aligned}
\end{equation*}
Then $\bar{U}$ can be re-written using $m$ and $D$:
\begin{equation*}
 \bar{U}=\frac{m}{D} \log \left( \frac{2}{\delta}\frac{m^2}{b_D^2}+1\right)
\end{equation*}
We define: $\bar{u}=D\bar{U}$ and $b_D= DB_D$. Then, we get: for $\epsilon\bar{U}\le(e-1)(1+\delta)B_D$,
\begin{equation}
\probability\left\{ \norm{F(\delta_i)} \ge \epsilon \right\} \le 2p\exp\left( -\frac{D\epsilon^2}{2(1+\delta)b_D+2\epsilon\bar{u}/3}\right)
\end{equation}
and for $\epsilon\bar{U}>(e-1)(1+\delta)B_D$,
\begin{equation}
\probability\left\{ \norm{F(\delta_i)} \ge \epsilon \right\} \le 2p\exp\left( -\frac{D\epsilon}{(e-1)\bar{u}} \right).
\end{equation}

\subsection{Union bound}
Now take the union bound over the centers of the $\epsilon$-net:
\begin{equation}
\label{eq:anchor_bound}
\probability\left\{ \bigcup_{i=1}^*\norm{F(\delta_i)} \ge \frac{\epsilon}{2}\right\} \le 4Tp \begin{cases} \exp\left( -\frac{\epsilon^2D}{8\left((1+\delta)b_D+\frac{2\epsilon}{6}\bar{u}\right)}\right) & \text{if}\enskip \epsilon\bar{U}_D\le(e-1)(1+\delta)B_D \\
\exp\left( -\frac{\epsilon D}{2(e-1)\bar{u}} \right) & \text{otherwise}.
\end{cases}
\end{equation}

\subsubsection[Optimizing over r]{Optimizing over $r$}
Combining \cref{eq:anchor_bound} and \cref{eq:Lipschitz_bound} and taking $\delta=1<2/(e-1)$ yields
\begin{equation*}
\probability \left\{ \sup_{\delta\in\mathcal{D}_{\bmC}} \norm{F(\delta)} \le \epsilon \right\} \ge 1 - \kappa_1 r^{-d} - \kappa_2 r^2,
\end{equation*}
with
\begin{equation*}
\kappa_2=4\sigma_p^2\epsilon^{-2}
\quad\text{and}\quad \kappa_1 = 2p(4l)^d\begin{cases} \exp\left( -\frac{\epsilon^2 D}{16\left(b_D+\frac{\epsilon}{6}\bar{u}_D\right)}\right) & \text{if}\enskip \epsilon\bar{U}_D\le 2(e-1)B_D \\
\exp\left( -\frac{\epsilon D}{2(e-1)\bar{u}_D} \right) & \text{otherwise}.
\end{cases}
\end{equation*}
we choose $r$ such that $d\kappa_1r^{-d-1}-2\kappa_2r=0$, i.e. $r=\left(\frac{d\kappa_1}{2\kappa_2}\right)^{\frac{1}{d+2}}$. Eventually let
\begin{equation*}
C'_d=\left(\left(\frac{d}{2}\right)^{\frac{-d}{d+2}}+\left(\frac{d}{2}\right)^{\frac{2}{d+2}}\right)
\end{equation*}
the bound becomes
\begin{equation*}
\begin{aligned}
\probability \left\{ \sup_{\delta\in\mathcal{D}_{\bmC}} \norm{\tilde{F}_i(\delta)} \ge \epsilon \right\}
&\le C'_d \kappa_1^{\frac{2}{d+2}}\kappa_2^{\frac{d}{d+2}} \\
&= C'_d\left(4\sigma_p^2\epsilon^{-2}\right)^{\frac{d}{d+2}} \left(2p(4l)^d\begin{cases} \exp\left( -\frac{\epsilon^2 D}{16\left(B_D+\frac{\epsilon}{6}\bar{U}_D\right)}\right) & \text{if}\enskip \epsilon\bar{U}_D\le 2(e-1)B_D \\ \exp\left( -\frac{\epsilon D}{2(e-1)\bar{U}_D} \right) & \text{otherwise} \end{cases} \right)^{\frac{2}{d+2}} \\
&= p C'_d 2^{\frac{2+4d+2d}{d+2}}\left( \frac{\sigma_p l}{\epsilon} \right)^{\frac{2d}{d+2}} \begin{cases} \exp\left( -\frac{\epsilon^2}{8(d+2)\left(B_D+\frac{\epsilon}{6}\bar{U}_D\right)}\right) & \text{if}\enskip \epsilon\bar{U}_D\le 2(e-1)B_D \\ \exp\left( -\frac{\epsilon}{(d+2)(e-1)\bar{U}_D} \right) & \text{otherwise} \end{cases} \\
&= p C'_d 2^{\frac{6d+2}{d+2}}\left( \frac{\sigma_p l}{\epsilon} \right)^{\frac{2}{1+2/d}} \begin{cases} \exp\left( -\frac{\epsilon^2}{8(d+2)\left(B_D+\frac{\epsilon}{6}\bar{U}_D\right)}\right) & \text{if}\enskip \epsilon\bar{U}_D\le 2(e-1)B_D \\ \exp\left( -\frac{\epsilon}{(d+2)(e-1)\bar{U}_D} \right) & \text{otherwise}. \end{cases}
\end{aligned}
\end{equation*}
Conclude the proof by taking $C_d=C'_d 2^{\frac{6d+2}{d+2}}$.

\section{Application of the bounds to decomposable, curl-free, divergence-free kernels}
\label{sec:bd_application}
\begin{proposition}[Bounding the term $\bf b_D$]\label{pr:variance-bound2}
Define the random matrix $\variance_{\mu}[A(\omega)]$ as follows: $\ell,m \inrange{1}{p}$, $\variance_{\mu}[A(\omega)]_{\ell m} = \sum_{r=1}^p \covariance{(A(\omega)_{\ell r}, A(\omega)_{r m}}$. For a given $\delta=x-z$, with the previous notations
\begin{equation*}
    b_D= D\norm{\expectation_{\mu}\left[\sum_{j=1}^D \left(\tilde{F}_j(\delta)\right)^2\right]}_2,
\end{equation*}
we have:
\begin{equation*}
  b_D \le \frac{1}{2} (\norm{(K_0(2
\delta)+K_0(0))\expectation_\mu[A(\omega)] - 2K_0(\delta)^2}_2 + 2\norm{\variance_{\mu}[A(\omega)]}_{2}).
\end{equation*}
\end{proposition}

\begin{proof}
We fix $\delta=x-z$. For sake of simplicity, we note: $B_D=\norm{\expectation_\mu\left[F_1(\delta)^2 + \ldots + F_D(\delta)^2\right]}_2$ and we have $b_D= D B_D$, with the notations of the theorem. Then
\begin{equation*}
\begin{aligned}
B_D&= \norm{\expectation_{\mu}\left[\sum_{j=1}^D\frac{1}{D^2} (\tilde{K}_j(\delta) - K_0(\delta))^2\right]}_2\\
&= \frac{1}{D^2}\norm{\sum_{j=1}^D \expectation_{\mu}\left[\left(\tilde{K}_j(\delta)^2 - \tilde{K}_j(\delta)K_0(\delta) - K_0(\delta)\tilde{K}_j(\delta) + K_0(\delta)^2\right)\right]}_2.
\end{aligned}
\end{equation*}
As $K_0(\delta)^*=K_0(\delta)$ and $A(\omega_j)^*=A(\omega_j)$, then $\tilde{K}_j(\delta)^*=\tilde{K}_j(\delta)$, we have
\begin{equation*}
B_D = \frac{1}{D^2}\norm{\sum_j \expectation_{\mu}\left[ \tilde{K}_j(\delta)^2 - 2 \tilde{K}_j(\delta)K_0(\delta) + K_0(\delta)^2\right]}_2.
\end{equation*}
From the definition of $\tilde{K}_j$, $\expectation_{\mu}[\tilde{K}_j(\delta)] = K_0(\delta)$ which leads to
\begin{equation*}
B_D= \frac{1}{D^2} \norm{\sum_{j=1}^D\expectation_{\mu}\left(\tilde{K}_j(\delta)^2- K_0(\delta)^2\right)}_2
\end{equation*}
Now we omit the $j$ index since all vectors $\omega_j$ are identically distributed and consider a random vector $\omega \sim \mu$:
\begin{equation*}
B_D = \frac{1}{D^2} \norm{D\expectation_\mu\left[(\cos\inner{\omega,\delta})^2A(\omega)^2\right]-K_0(\delta)^2}_2
\end{equation*}
A trigonometry property gives us: $(\cos \inner{\omega,\delta})^2 = \frac{1}{2}(\cos \inner{\omega,2\delta} + \cos \inner{\omega,0})$
\begin{equation}\label{eq:variance}
\begin{aligned}
B_D&= \frac{1}{D^2}\norm{\frac{D}{2}\expectation_\mu\left[(\cos\inner{\omega,2\delta} + \cos \inner{\omega,0})A(\omega)^2\right] - K_0(\delta)^2}_2\\
&= \frac{1}{2D}\norm{\expectation_\mu\left[(\cos\inner{\omega,2\delta} + \cos \inner{\omega,0})A(\omega)^2\right] - \frac{2}{D}K_0(\delta)^2}_2\\
\end{aligned}
\end{equation}
Moreover, we write the expectation of a matrix product, coefficient by coefficient, as: $\forall \ell, m \inrange{1}{p}$,
\begin{equation*}
\begin{aligned}
\expectation_{\mu}\left[\left(\cos \inner{\omega,2\delta} A(\omega)^2\right)\right]_{\ell m} &= \sum_{r} \expectation_{\mu}\left[ \cos \inner{\omega,2\delta}A(\omega)\right]_{\ell r}\expectation_{\mu}\left[ A(\omega)\right]_{r m} + \covariance{\cos \inner{\omega,2\delta}A(\omega)_{\ell r}, A(\omega)_{r m}}\\
\expectation_\mu\left[\left(\cos \inner{\omega,2\delta} A(\omega)^2\right)\right] &= \expectation_{\mu}[\cos \inner{\omega,2\delta}A(\omega)]\expectation_{\mu}[A(\omega)] + \Sigma^{\cos}\\
&= K_0(2\delta)\expectation_{\mu}[A(\omega)] + \Sigma^{\cos}
\end{aligned}
\end{equation*}
where the random matrix $\Sigma^{\cos}$ is defined by: $\Sigma_{\ell m}^{\cos} = \sum_r \covariance{\cos \inner{\omega,2\delta}A(\omega)_{\ell r}, A(\omega)_{r m}}$. Similarly, we get:
\begin{equation*}
\begin{aligned}
\expectation_{\mu}\left[\cos \inner{\omega,0} A(\omega)^2\right]&=K_0(0)\expectation_\mu\left[A(\omega)\right] + \Sigma^{\cos}.
\end{aligned}
\end{equation*}
Hence, we have:
\begin{equation*}
\begin{aligned}
B_D &= \frac{1}{2D} \norm{(K_0(2 \delta)+K_0(0))\expectation_\mu[A(\omega)] -2K_0(\delta)^2+ 2 \Sigma^{\cos}}_2 \\
&\leq \frac{1}{2D}\left[\norm{\left(K_0(2 \delta)+K_0(0)\right)\expectation_\mu[A(\omega)] - 2 K_0(\delta)^2}_2 +2\norm{\variance_{\mu}[A(\omega)]}_2 \right],\label{eq:last-inequality}
\end{aligned}
\end{equation*}
using $\norm{\Sigma^{\cos}}_2 \leq \norm{\variance_{\mu}[A(\omega)]}_2$, where $\variance_{\mu}[A(\omega)]= \expectation_{\mu}[(A(\omega) - \expectation_{\mu}[A(\omega)])^2]$ and for all $\ell,m \inrange{1}{p}$,  $\variance_{\mu}[A(\omega)]_{\ell m} = \sum_{r=1}^p \covariance{A(\omega)_{\ell r}, A(\omega)_{r m}}$.
\end{proof}
For the three kernels of interest, we illustrate this bound in \cref{fig:variance_error}.

\begin{figure}[!ht]
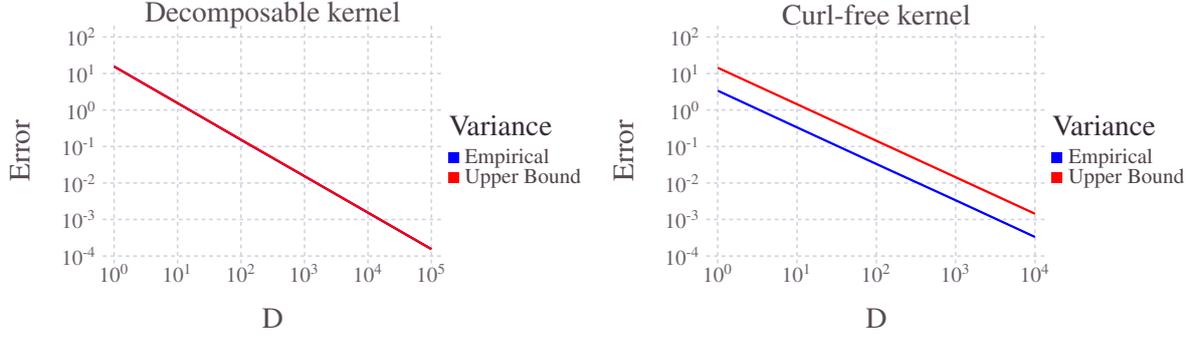

  \centering
  \begin{tabular}{cc}
  \includegraphics[width=0.47\textwidth]{Figures/variance_dec.tikz} & \includegraphics[width=0.47\textwidth]{Figures/variance_curl.tikz}
  \end{tabular}
  \caption{Upper bound on the variance of the decomposable and curl-free kernel obtained via ORFF. We generated a random point in $[-1,1]^4$ and computed the empirical variance of the estimator (blue line). We also plotted (red line) the theoretical bound proposed in \cref{pr:variance-bound2}.}
  \label{fig:variance_error}
\end{figure}

\section{Additional information and results}
\subsection{Implementation detail}
\label{subsec:implementation_detail}
 For each $\omega_j \sim \mu$, let $B(\omega_j)$ be a $p$ by $p'$ matrix such that $B(\omega_j)B(\omega_j)^*=A(\omega_j)$. In practice, making a prediction $y=h(x)$ using directly the formula $h(x)=\tilde{\Phi}(x)^*\theta$ is prohibitive. Indeed, if $\Phi(x)=\Vect_{j=1}^D\exp(-i\inner{x,\omega_j})B(\omega_j)^*B(\omega_j)$, it would cost $O(Dp'p)$ operation to make a prediction, since $\tilde{\Phi(x)}$ is a $Dp'$ by $p$ matrix.

\subsubsection{Minimizing Eq. 11 in the main paper}
Recall we want to minimize
\begin{equation}
  \label{eq:stein1}
  \theta_* = \argmin_{\theta\in\mathbb{R}^{p'D}}\norm{ \tilde{\phi}(X)^*\theta -Y}^2+\lambda\norm{\theta}^2.
\end{equation}
The idea is to replace the expensive computation of the matrix-vector product by $\tilde{\Phi}(X)^*\theta$ by a cheaper linear operator $P_x$ such that $\tilde{\Phi}(X)^*\theta=P_x\theta$. In other word, we minimize:
\begin{equation}
  \label{eq:stein2}
  \theta_* = \argmin_{\theta\in\mathbb{R}^{p'D}}\norm{ P_X\theta -Y}^2+\lambda\norm{\theta}^2.
\end{equation}
Among many possibilities of solving \cref{eq:stein2}, we focused on two types of methods:
\begin{enumerate}[i)]
    \item Gradient based methods: to solve \cref{eq:stein2} one can iterate $\theta_{t+1}=\theta_t - \eta_t(P_X^*(P_X\theta_t-y)+\lambda\theta_t)$. replace $P_X$ by $P_{x_t}$, where $x_t$ is a random sample of $X$ at iteration $T$ to perform a stochastic gradient descent.
    \item Linear methods: since the optimization problem defined in \cref{eq:stein2} is convex, one can find a solution to the first order condition, namely $\theta_*$ such that $(P_X^*P_X) \theta_*=P_X^*y$. Many Iterative solvers able to solve such linear system are available, such as \citet{sonneveld2008idr} or \citet{fong2011lsmr}.
\end{enumerate}

\subsubsection{Defining efficient linear operators}
\paragraph{Decomposable kernel}
Recall that for the decomposable kernel $K_0(\delta)=k_0(\delta)A$ where $k_0$ is a scalar shift-invariant kernel, $A(\omega_j)=A=BB^*$ and
\begin{equation*}
\begin{aligned}
    \tilde{\Phi}(x)&=\Vect_{j=1}^D\exp(-i\inner{x,\omega_j})B^* \\
          &=\tilde{\phi}(x)\otimes B^*
\end{aligned}
\end{equation*}
where $\tilde{\phi}(x)=\vect_{j=1}^D\exp(-i\inner{x,\omega_j})$ is the RFF corresponding to the scalar kernel $k_0$. Hence $h(x)$ can be rewritten in the following way
\begin{equation*}
\begin{aligned}
    h(x)&=(\tilde{\phi}(x)\otimes B^*)^*\Theta
        &=\text{vec}(\tilde{\phi}(x)^*\Theta B^*)
\end{aligned}
\end{equation*}
where $\Theta$ is a $D$ by $p'$ matrix such that $\text{vec}(\Theta)=\Theta$. Eventually we define the following linear (in $\theta$) operator
\begin{equation*}
    P^\text{dec}_x:\theta\mapsto \text{vec}(\tilde{\phi}(x)^*\Theta B^*)
\end{equation*}

Then
\begin{equation*}
    h(x)=P^\text{dec}_x \theta=\tilde{\Phi}(x)^*\theta
\end{equation*}
Using this formulation, it only costs $O(Dp'+p'p)$ operations to make a prediction. If $B=I_d$ it reduces to $O(Dp)$. Moreover this formulation cuts down memory consumption from $O(Dp'p)$ to $O(D+p'p)$.

\paragraph{Curl-free kernel}
For the Gaussian curl-free kernel we have, $K_0(\delta)=-\nabla\nabla^Tk_0(\delta)$ and the associated feature map is $\Phi(x)=\Vect_{j=1}^D\exp(-i\inner{x,\omega_j})\omega_j^*$. In the same spirit we can define a linear operator
\begin{equation*}
    P^\text{curl}_x:\theta\mapsto \text{vec}\left(\sum_{j=1}^D\tilde{\phi}(x)_j^*\Theta_j \omega_j\right),
\end{equation*}
such that $h(x)=P^\text{curl}_x \theta=\tilde{\Phi}(x)^*\theta$. Here the computation time for a prediction is $O(Dp)$ and uses $O(D)$ memory.

\paragraph{Div-free kernel}
For the Gaussian divergence-free kernel, $K_0(\delta)=(\nabla\nabla^T-I\Delta)k_0(\delta)$ and $\Phi(x)=\Vect_{j=1}^D\exp(-i\inner{x,\omega_j})(I-\omega_j^*\omega_j)^{1/2}$. Hence, we can define a linear operator
\begin{equation*}
    P^\text{div}_x:\theta\mapsto \text{vec}\left(\sum_{j=1}^D\tilde{\phi}(x)_j^*\Theta_j (I_d-\omega_j\omega_j^*)^{1/2}\right),
\end{equation*}
such that $h(x)=P^\text{div}_x \theta=\tilde{\Phi}(x)^*\theta$. Here the computation time for a prediction is $O(Dp^2)$ and uses $O(Dp^2)$ memory.



\begin{table}[ht]
    \centering
    \resizebox{0.9\textwidth}{!} {
    \begin{tabular}{cccc} \hline
        Feature map & $P_x:\theta\mapsto$ & $P_x^*:y\mapsto$ & $P_x^*P_x:\theta\mapsto$ \\ \hline \\
        $\Phi^{dec}(x)$ & $ \text{vec}(\tilde{\phi}(x)^*\Theta B^*)$ & $\text{vec}(\tilde{\phi}(x)y^*B)$ & $ \text{vec}(\tilde{\phi}(x)\tilde{\phi}(x)^*\Theta B^*B)$\\

        $\Phi^{curl}(x)$ & $\text{vec}\left(\sum_{j=1}^D\tilde{\phi}(x)_j^*\Theta_j \omega_j\right)$ & $\vect_{j=1}^D\tilde{\phi}(x)_jy^*\omega_j$ & $\text{vec}\left(\tilde{\phi}(x)\tilde{\phi}(x)^* \left(\vect_{j=1}^D\Theta_j\norm{\omega_j}^2\right)\right)$ \\

        $\Phi^{div}(x)$ & $\text{vec}\left(\sum_{j=1}^D\tilde{\phi}(x)_j^*\Theta_j (I_d-\omega_j\omega_j^*)^{1/2}\right)$ & $\vect_{j=1}^D\tilde{\phi}(x)_jy^*(I_d-\omega_j\omega_j^*)^{1/2}$ & $\text{vec}\left(\tilde{\phi}(x)\tilde{\phi}(x)^* \left(\vect_{j=1}^D\Theta_j(I_d-\omega_j\omega_j^*)\right)\right)$ \\ \hline
    \end{tabular}
    }
    \caption{fast-operator for different Feature maps.}
    \label{tb:fast-op}
\end{table}

\begin{table}[ht]
    \centering
    \begin{tabular}{cccc} \hline
        Feature map & $P_x$ & $P_x^*$ & $P_x^*P_x$ \\ \hline \\
        $\Phi^{dec}(x)$ & $O(D+pp')$ & $O(Dp'+pp')$ & $O(D^2+Dp'^2)$\\

        $\Phi^{curl}(x)$ & $O(Dp)$ & $O(Dp)$ & $O(D^2p+Dp)$\\

        $\Phi^{div}(x)$ & $O(Dp^2)$ &  $O(Dp^2)$ & $O(D^2p+Dp^2)$ \\ \hline
    \end{tabular}
    \caption{Time complexity to compute different Feature maps with fast-operators (for one point $x$).}
    \label{Optimized operator for different Feature maps.}
\end{table}



\subsection{Simulated dataset}
\label{sec:more_simulation}
\begin{figure}[H]
\centering
\begin{tabular}{cc}
\includegraphics[trim=1.8cm 1cm 2cm 1cm,width=.45\textwidth,clip=true]{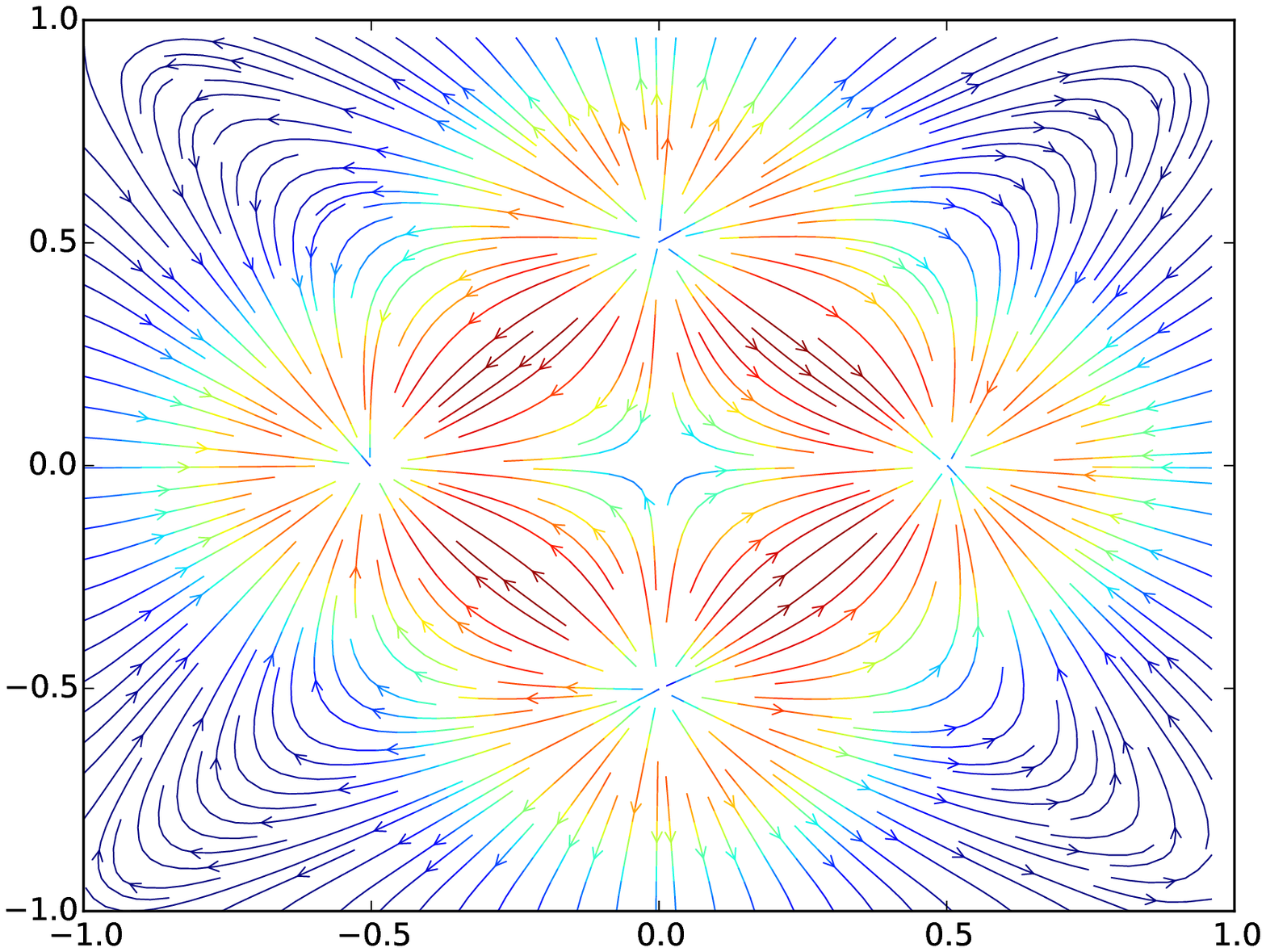} &
\includegraphics[trim=1.8cm 1cm 2cm 1cm,width=.45\textwidth,clip=true]{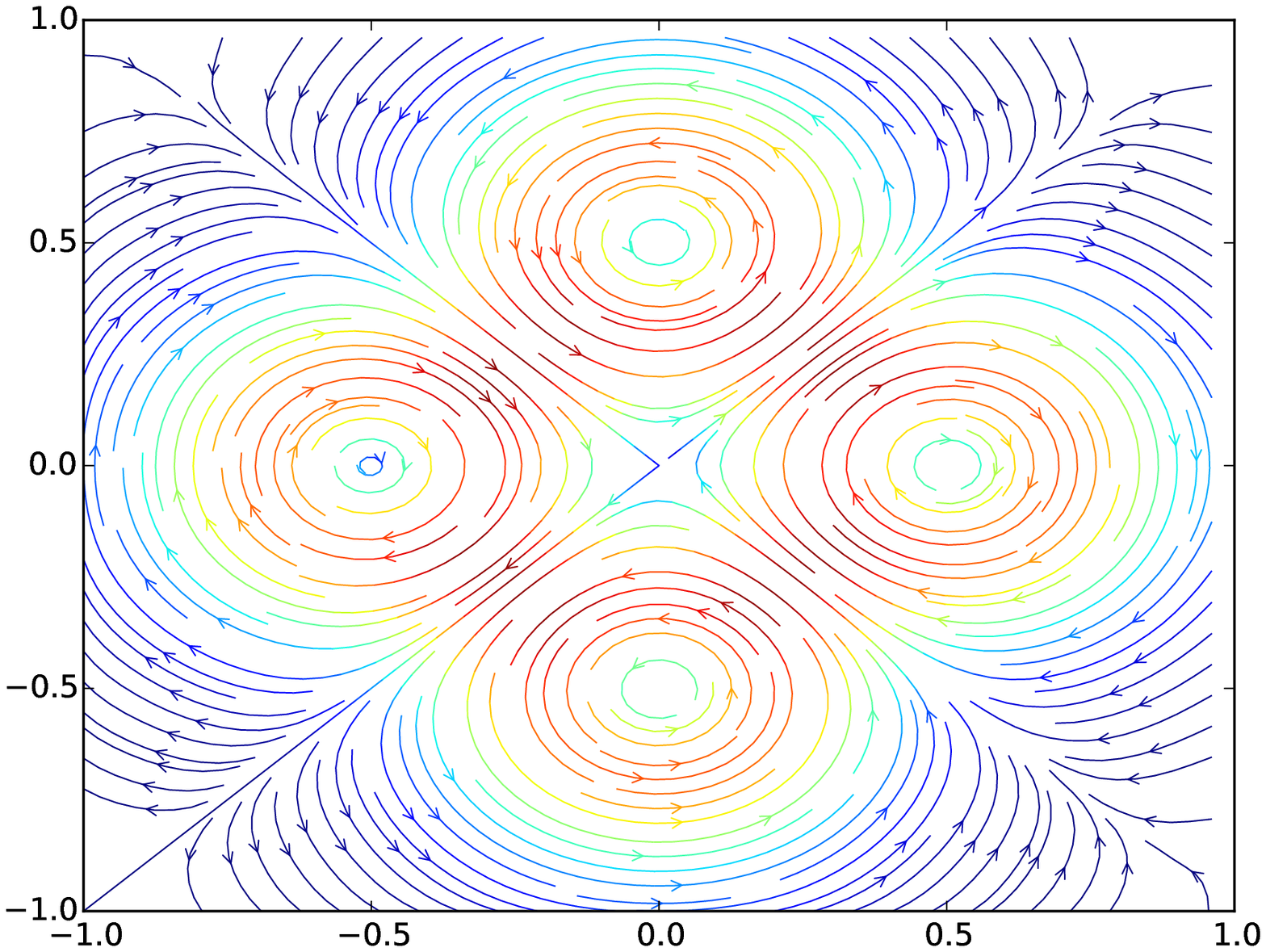} \\
Curl-free & Divergence-free
\end{tabular}
\caption{Synthetic data used to train the curl-free and divergence free ORFF and OVK.}
\label{fig:curl-div-fields}
\end{figure}

We also tested our approach when the output data are corrupted with a Gaussian noise with arbitrary covariance. Operator-valued kernels-based models as well as their approximations are in this case more appropriate than independent scalar-valued models. We generated a dataset $\mathcal{A}_{\text{dec}}^N=(\mathcal{X},\mathcal{Y})$ adapted to the decomposable kernel with $N$ points $x_i\in\mathcal{X}\subset\mathbb{R}^{20}$ to $y_i\in\mathcal{Y}\subset\mathbb{R}^{20}$, where the outputs have a low-rank.
The inputs where drawn randomly from a uniform distribution over the hypercube $\mathcal{X}=[-1;1]^{20}$. To generate the outputs we constructed an ORFF model from a decomposable kernel $K_0(\bfdelta)=Ak_0(\bfdelta)$, where $A$ is a random positive semi-definite matrix of size $20\times 20$, rank $1$ and $\norm{A}_2=1$ and $k_0$ is a Gaussian kernel with hyperparameter $\gamma=1/(2\sigma^2)$.
We choose $\sigma$ to be the median of the pairwise distances over all points of $\mathcal{X}$ (Jaakkola's heuristic \cite{jaakkola1999using}).
Then we generate a parameter vector for the model $\theta$ by drawing independent uniform random variable in $[-1;1]$ and generate $N$ outputs $y_i=\tilde{\Phi}_D(x_i), x_i \in \mathcal{X}, y_i \in\mathcal{Y}, i\in\{1\hdots N\}$.
We chose $D=10^4$ relatively large to avoid introducing too much noise due to the random sampling of the Fourier Features.
We compare the exact kernel method OVK with its ORFF approximation on the dataset $\mathcal{A}^{n}_{\text{dec}}$ with additive non-isotropic Gaussian noise: $y^{\text{noise}}_i=\tilde{\Phi}_D(x_i)+\epsilon_i$ where $\epsilon_i\sim\mathcal{N}(0,\Sigma)$ and $\norm{\Sigma}_2=\sqrt{\variance[y_i]}$.
We call the noisy dataset $\mathcal{A}^{N}_{\text{dec,noise}}=(\mathcal{X},\mathcal{Y}_{\text{noise}}){\text{dec,noise}})$. The results are given in \cref{tab:RMSE_ovk_vs_orff}, where the reported error is the root mean-squared error.
\begin{table}[!ht]
\centering
\resizebox{0.9\textwidth}{!} {
\begin{tabular}{ccccc}
N & ORFF & OVK & ORFF NOISE & OVK NOISE \\ \hline \\
 $10^2$ & $9.21\cdot10^{-2}\pm 4\cdot10^{-3}$ & $4.36\cdot10^{-2}\pm 7\cdot10^{-3}$ & $1.312\cdot10^{-1}\pm 1\cdot10^{-2}$& $1.222\cdot10^{-1}\pm 9\cdot10^{-2}$\\
 $10^3$ & $5.97\cdot10^{-2}\pm 2\cdot10^{-4}$ & $2.13\cdot10^{-2}\pm 8\cdot10^{-4}$ & $1.085\cdot10^{-1}\pm 2\cdot10^{-2}$ & $0.990\cdot10^{-1}\pm 4\cdot10^{-2}$\\
 $10^4$ & $3.56\cdot10^{-2}\pm 5\cdot10^{-4}$ & $1.01\cdot10^{-2}\pm 1\cdot10^{-4}$ & $.876\cdot10^{-1}\pm 3\cdot10^{-3}$ & $0.825\cdot10^{-1}\pm 2\cdot10^{-3}$\\
 $10^5$ & $2.89\cdot10^{-2}\pm 7\cdot10^{-4}$ & $N/A$ & $.717\cdot10^{-1}\pm 3\cdot10^{-3}$ & $N/A$\\
\end{tabular} }
\caption{RMSE, average of 10 runs on synthetic data.}
\label{tab:RMSE_ovk_vs_orff}
\end{table}

\end{document}